\documentclass[12pt,letterpaper]{amsart}
\usepackage{hyperref}
\usepackage{graphicx, color}
\usepackage{mathtools}
\usepackage{enumerate,enumitem}
\usepackage{bbm}
\usepackage[margin=1in]{geometry}
\usepackage{amssymb}
\usepackage[linesnumbered,ruled,vlined]{algorithm2e}
\usepackage{subcaption}
\usepackage{float}
\usepackage{tikz}
\usepackage{yhmath}
\usepackage{multirow}
\usepackage[numbers]{natbib}
\usepackage{tabu}
\usepackage{tabularx}
\usepackage{algorithm2e}
\usepackage{algorithmic}

\usepackage{mathabx}

\usepackage[flushleft]{threeparttable}
\usepackage{array,booktabs,makecell}
\usepackage{nicematrix}
\usepackage{arydshln}

\usepackage[capitalize,noabbrev]{cleveref}

\newtheorem{theorem}{Theorem}[section]

\newtheorem{lemma}[theorem]{Lemma}
\newtheorem{corollary}[theorem]{Corollary}
\theoremstyle{definition}
\newtheorem{definition}[theorem]{Definition}

\theoremstyle{remark}
\newtheorem{remark}[theorem]{Remark}

\DeclareMathOperator{\argmin}{argmin}
\DeclareMathOperator{\sign}{sign}

\newcommand{\R}{\mathbb{R}}

\newcommand{\RSnote}[1]{\textcolor{red}{[{\em {\bf **RS Note:} #1}]}}

\SetKwInput{KwInput}{Input}                
\SetKwInput{KwOutput}{Output}              
\SetKwInput{KwIni}{Initialize Parameters}

\author{Shihao Zhang}
\address{Department of Mathematics, University of California San Diego}
\email{shz051@ucsd.edu}

\author{Rayan Saab}
\address{Department of Mathematics and Hal{\i}c{\i}o\u{g}lu Data Science Institute, University of California San Diego}
\email{rsaab@ucsd.edu}

\title{Theoretical Guarantees for Low-Rank Compression of Deep Neural Networks}

\begin{document}

\begin{abstract}
Deep neural networks have achieved state-of-the-art performance across numerous applications, but their high memory and computational demands present significant challenges, particularly in resource-constrained environments. Model compression techniques, such as low-rank approximation, offer a promising solution by reducing the size and complexity of these networks while only minimally sacrificing accuracy. In this paper, we develop an analytical framework for data-driven post-training low-rank compression. We prove three recovery theorems under progressively weaker assumptions about the approximate low-rank structure of activations, modeling deviations via noise. Our results represent a step toward explaining why data-driven low-rank compression methods outperform data-agnostic approaches and towards theoretically grounded compression algorithms that reduce inference costs while maintaining performance.
\end{abstract}
\maketitle

\section{Introduction}\label{sec:intro}


Over the past decade, deep neural networks (DNNs) have achieved remarkable success across a wide range of applications, with convolutional neural networks (CNNs) excelling in computer vision and transformers revolutionizing natural language processing. However, these achievements come at the cost of significant memory and computational demands, primarily due to the highly over-parameterized nature of modern neural networks. Such models require substantial memory to store their weights and considerable computational resources for inference. Consequently, the demand for model compression techniques has grown, particularly in contexts where storage efficiency and adaptability to mobile devices are crucial \cite{deng2020model,  wang2016accelerating}. The urgency of this challenge has been amplified by the growing focus on compressing large language models, which has become an area of intense research interest \cite{xu2023survey, zhu2023survey}.



\subsection{Setting and notation}\label{sec:notation} To explain the challenges and opportunities associated with neural network compression, let us introduce a standard neural network model, namely the  $L$-layer multi-layer perceptron. An $L$-layer multi-layer perceptron is a function $\mathbf{\Phi}: \R^{N_0} \to \R^{N_L}$ that acts on a sample of data $x\in \mathbb{R} ^ {N_0}$ via successive compositions of affine and non-linear functions:
\begin{equation}
    \mathbf{\Phi}(x):=\phi^{(L)}\circ A^{(L)}\circ \cdots \phi^{(1)}\circ A^{(1)}(x).
\end{equation}
 Here each $\phi^{(i)}:\mathbb{R} ^ {N_i} \longrightarrow \mathbb{R} ^ {N_i}$ is a nonlinear ``activation" function, with a popular choice being the ReLU activation function $\phi^{(i)}=\rho$. With a slight abuse of notation ReLU acts elementwise via $$\rho(x) = \begin{cases}
x, & \text{if } x \geq 0 \\
0, & \text{otherwise}
\end{cases}.$$ Meanwhile, each $A^{(i)}:\mathbb{R} ^ {N_{i-1}} \longrightarrow \mathbb{R} ^ {N_i}$ is simply an affine map given by $A^{(i)}(z)={W^{(i)}}^\top z+b^{(i)}$. Here, $W^{(i)}\in \mathbb{R} ^ {N_{i-1}\times N_{i}}$, $i=1, \ldots, L$, are weight matrices, $b^{(i)}\in \mathbb{R} ^ {N_i}$ are bias vectors. We call $z = \phi^{(i-1)}A^{(i-1)}\circ \cdots\circ \phi^{(1)}A^{(1)}(x)$ the activation of the $(i-1)$th layer associated with an input $x$, and $A^{(i)}(z)$ the pre-activation of the $i$-th layer. Consequently,   $\phi^{(i)}\circ A^{(i)}(z) $ is the activation of the $i$-th layer.
By adding a coordinate $1$ to $z$ and treating $b^{(i)}$ as an extra row appended to the weight matrix $W^{(i)}$, we can henceforth ignore the bias terms in our analysis without loss of generality. 

Given a data set $X_0 \in \mathbb{R} ^ {m\times N_{0}}$ with vectorized data stored as rows, and a
trained neural network $\mathbf{\Phi}$ with weight matrices $W^{(i)}$,
let $\mathbf{\Phi}^{(i)}$
denote the original network
truncated after layer $i$. The resulting  {activations} from the $i$-th layer
are then $X^{(i)}:=\mathbf{\Phi}^{(i)}(X_0) =\phi^{(i)}(X^{(i-1)}W^{(i)})$, while $X^{(i-1)}W^{(i)}$ are the associated \emph{pre-activations}.
For notational convenience, we define $X^{(0)}=X_0$. Furthermore, we assume $X_0$ is data chosen from a separate dataset independent of the training data used to initially train the parameters $W^{(i)}, i=1,2,.......,L$. The infinity norm of a matrix $\|\cdot\|_{\infty}$ always refers to the element-wise $\ell_\infty$-norm in this paper.

\subsection{Background and motivation}\label{subsec:Background}

Common approaches for compressing deep neural networks include low-rank approximation \cite{jaderberg2014speeding, zhang2015accelerating, chen2021drone}, pruning or sparsification \cite{hassibi1993optimal, han2015learning, kitaev2020reformer}, quantization \cite{zhang2023post, zhang2023spfq, jacob2018quantization}, and knowledge distillation \cite{liang2023homodistil, choudhary2020comprehensive, neill2020overview}. Among these, low-rank decomposition reduces the number of parameters in an $L$-layer neural network by replacing weight matrices $W^{(i)} \in \R^{N_{i-1} \times N_i}$ with a product of low-rank matrices. This reduces the parameter count for layer $i$ from $N_{i-1}N_i$ to $r_i(N_{i-1} + N_i)$, where $r_i \ll \min\{N_{i-1}, N_i\}$ denotes the rank of the approximating matrix.
This not only reduces the amount of memory needed to store these fewer parameters, but also accelerates inference due to the reduced cost of matrix multiplication.


A straightforward approach to low-rank approximation involves using the singular value decomposition (SVD) to replace  weight matrices $W^{(i)}$ with a product of low-rank factors. While conceptually simple, this method often yields suboptimal results unless followed by fine-tuning, which essentially involves re-training the low-rank factors \cite{denton2014exploiting, lebedev2014speeding, jaderberg2014speeding}. In contrast, data-driven low-rank approximation algorithms make use of a sample of input data to guide the neural network compression. These data-driven methods tend to perform well in practice, even before fine-tuning, and typically require less extensive fine-tuning than their data-agnostic counterparts as documented, for example, in \cite{zhang2015accelerating,li2018constrained, chen2021drone}. Indeed, numerical evidence presented in \cite{yu2023compressing} demonstrate that the pre-activations $X^{(i-1)}W^{(i)}$ often exhibit more pronounced low-rank characteristics than the weight matrix $W^{(i)}$ itself. In turn, heuristically, this suggests that by approximating $W^{(i)}$ with a low-rank matrix that preserves the important singular values of $X^{(i-1)}W^{(i)}$ rather than $W^{(i)}$, one can obtain better performance. 


Despite the observed advantages of data-dependent methods, they share certain limitations of data-agnostic methods. For instance, they rarely explicitly account  for the non-linear activation functions in the network and are generally not supported by rigorous theoretical guarantees.

In this paper, we develop an analytical framework that theoretically justifies why incorporating input data in post-training low-rank compression yields a better compressed model compared to data-agnostic approaches. This may help clarify why such methods provide a better initialization for fine-tuning, resulting in reduced fine-tuning time and improved approximation of the original network. As alluded to above, a central motivating theme in our framework is the observation that existing data-dependent algorithms primarily focus on minimizing the reconstruction error of the (pre-)activations during weight matrix compression.

A significant challenge in explaining the effectiveness of low-rank compression algorithms lies in the fact that weight matrices from pretrained models are typically not exactly low-rank. This makes it difficult to use traditional approximation error bounds to justify their performance, contributing to the scarcity of theoretical analysis and error bounds in the existing literature on low-rank compression. To address this, we observe that a low-rank approximation problem—defined by minimizing the Frobenius norm under a nuclear norm constraint—can often be interpreted as the dual formulation of a low-rank recovery problem, where the objective is to minimize the nuclear norm under a Frobenius norm constraint. This perspective allows us to reframe the low-rank weight approximation problem as a low-rank recovery task, assuming the existence of an underlying unknown (approximately) low-rank model. In this framework, the pretrained neural network can be viewed as a noisy observation of the underlying low-rank model, and the goal of compression is to recover the low-rank model by minimizing the reconstruction error of the (pre-)activations. 

\subsection{Contributions}\label{sec:contribution}

We establish three low-rank recovery theorems under realistic and increasingly weaker assumptions, leveraging techniques from compressed sensing theory and matrix algebra to show that approximately accurate recovery is achievable within our proposed framework. To the best of our knowledge, these recovery theorems represent the first formal attempt to provide theoretical support for the design of data-driven, post-training low-rank compression methods. Below, we summarize our main results. Here, one should think of $X$ as the input activations from the previous layer in the pre-trained model $\mathbf{\Phi}$, and  of $\widecheck{X}$ as the corresponding input activations for the low-rank model $\widecheck{\mathbf{\Phi}}$.

\begin{theorem}(Abridged version of \cref{thm:recovery one})\label{thm:short one}
    Let $X, \widecheck{X} \in \mathbb{R}^{d_1 \times d}$, $d_1\geq d$ be full rank  and $W\in \mathbb{R}^{d \times d_2}$. Assume there exists a rank-$r$ matrix $M\in \mathbb{R}^{d \times d_2}$ such that $\|XW - (\widecheck{X}M+G)\|_{op}^2\leq \epsilon d_1$, where $G\in \mathbb{R}^{d_1 \times d_2}$ is a zero-mean sub-Gaussian matrix with i.i.d entries of variance $\sigma^2$. Then, 
    $\hat{M}:=\underset{\mathrm{rank}(Z)\leq r}{\argmin} \|XW-\widecheck{X}Z\|_F$ 
    satisfies $$\frac{\|\widecheck{X}M-\widecheck{X}\hat{M}\|_F^2}{d_1 d_2}\lesssim r\cdot \frac{d_2+d}{d_1 d_2}\sigma^2 +\epsilon$$ with high probability.
\end{theorem}


This theorem implies that if the pre-activations in a layer of the original neural network can be represented as a noisy version of those from an underlying ``compressed" network (i.e., a network with a low-rank weight matrix), then the low-rank matrix can be approximately recovered by solving a minimization problem. Moreover, the mean squared error decreases linearly with the dimensions, and it is noteworthy that the solution to this optimization problem can be efficiently computed using an appropriate singular value decomposition (as can be seen from the proof).

\begin{theorem}(Abridged version of \cref{cor:recovery two})\label{thm:short two}
    Let $X, \widecheck{X} \in \mathbb{R}^{d_1 \times d}$ and $W\in \mathbb{R}^{d \times d_2}$. Suppose there exists $M\in \mathbb{R}^{d \times d_2}$ such that $\widecheck{X}M$ is approximately rank-$r$ (\cref{def:approx lr}) and $\|XW - (\widecheck{X}M+G)\|_F^2\leq \epsilon d_1 d_2$, where $G\in \mathbb{R}^{d_1 \times d_2}$ has independent zero-mean bounded random entries. We assume $\|\widecheck{X}M\|_{\infty}\leq \alpha$ and $\|G\|_{\infty}\leq \beta$.
    Let $\Omega:=\{N: N\in \mathbb{R}^{d \times d_2}, \|\widecheck{X}N\|_*\leq\alpha \sqrt{r d_1 d_2};\ \|\widecheck{X}N\|_{\infty}\leq \alpha \}$.
    Then, minimizing the linear reconstruction $\hat{M}\in\underset{Z\in \Omega}{\argmin} \|XW-\widecheck{X}Z\|_F$ ensures that the mean square error satisfies: $$\frac{\|\widecheck{X}M-\widecheck{X}\hat{M}\|_F^2}{d_1 d_2}\lesssim (\alpha^2+\alpha\beta) \sqrt{\frac{r(d_1+d_2)}{d_1 d_2}}+\epsilon$$ with high probability.
\end{theorem}


This theorem can be interpreted similarly to the previous one, but with a weaker assumption, namely of the underlying matrix being approximately low-rank. Consequently, the squared error exhibits sub-linear decay with respect to the dimensions, specifically at a square root rate. Additionally, the optimization problem becomes more challenging due to the constraint, as no simple explicit solution is available. However, it remains a convex problem that can be solved using existing convex programming techniques.

\begin{theorem}(Abridged version of \cref{thm:nonlinear recovery})\label{thm:short nonl}
    Let $\widecheck{X} \in \mathbb{R}^{d_1 \times d}$, $d_1\geq d$ be full rank and $M\in \mathbb{R}^{d \times d_2}$ such that $\widecheck{X}M$ is approximately rank-$r$. Let $Z = \rho(\widecheck{X}M+G)$, where $G\in \mathbb{R}^{d_1 \times d_2}$ is a random Gaussian matrix with i.i.d $\mathcal{N}(0, \sigma^2)$ and $\rho$ is ReLU function acting entry-wise. Also, assume $\|\widecheck{X}M\|_{\infty}\leq \alpha$.  Then the solution $\hat{M}$ to the convex programming problem (\ref{opt P star prime}), that involves $Z$, ensures that
    the mean square error satisfies: $$\frac{\|\widecheck{X}M-\widecheck{X}\hat{M}\|_F^2}{d_1 d_2}\lesssim_{\alpha, \sigma} \sqrt{\frac{r(d_1+d_2)\log(d_1d_2)}{d_1 d_2}}$$ with high probability.
\end{theorem}


This theorem takes a step further by incorporating the non-linear ReLU activation into the framework and addressing unbounded Gaussian noise. The optimization problem becomes more complex, while the squared error remains essentially the same order as in the previous theorem, with an additional logarithmic term to account for the potentially unbounded noise.


\subsection{Limitations}

Explicitly introducing the non-linearity into low-rank approximation algorithms for neural network compression has been observed to reduce the accuracy drop, even in the absence of fine-tuning (e.g., \cite{zhang2015accelerating}). However, our nonlinear recovery theorem does not reflect this benefit in the error bound. Additionally, from an algorithmic perspective, directly addressing the (ReLU) activation function without relying on convex relaxation remains an open problem.

\subsection{Organization}
In the following sections we will 
prove our main theorems, with  \cref{sec:strong thm} focusing on the proof of \cref{thm:short one} and \cref{sec:weak theorem} on \cref{thm:short two} and \cref{thm:short nonl}. Meanwhile, the complete proof of the non-linear recovery theorem is in the appendix due to its technical nature. Some comments on the convex relaxation used in non-linear recovery theorem are provided in \cref{sec:con to frob}.

\section{Related Work}\label{sec:rel work}

The problem of low-rank approximation and recovery has been extensively studied, particularly in the context of compressed sensing. Foundational works include \cite{fazel2002matrix, fazel2008compressed, jain2010guaranteed}, among others. 

The standard low-rank matrix recovery (LRMR) task is to recover a matrix ${X}_0 \in \mathbb{R}^{m \times n}$, say of rank $r$,  from observations $y = \mathcal{A}({X}_0) + z$, where $z$ denotes noise. Here, $\mathcal{A}: \mathbb{R}^{m \times n} \to \mathbb{R}^L$ is a linear measurement operator, which often acts on ${X}_0$ through inner products with $L$ matrices $A_1, \dots, A_L \in \mathbb{R}^{m \times n}$ \cite{davenport2016overview}. A specific instance arises when these matrices are elementary, reducing the problem to low-rank matrix completion (LRMC). In LRMC, the goal is to (approximately) recover ${X}_0$ from a subset of observed entries, indexed by $\Omega \subset N$. Observations are modeled as $P_{\Omega}({X}_0)$, possibly perturbed by noise, where $P_{\Omega}$ is the associated projection operator \cite{nguyen2019low}.

These problems can be formulated as optimization tasks. For instance, LRMR can be posed as $\underset{X}{\text{min}} \|\mathcal{A}(X) - y\|_2^2$, while LRMC is often expressed as $\underset{X}{\text{min}} \|P_{\Omega}(X) - P_{\Omega}({X}_0)\|_2^2$, subject to the constraint $\operatorname{rank}(X) \leq r$ (if $r$ is known). However, minimizing under a low-rank constraint is generally NP-hard \cite{gillis2011low}. Thus, nuclear norm minimization is often used instead, with the associated optimization problems \[\underset{X}{\text{min}} \|X\|_* \  \text{ subject to} \ \mathcal{A}(X) \approx y\] for LRMR or its analog with $ P_{\Omega}$ replacing $\mathcal{A}$ for LRMC. Works in this area \cite{jain2013low, recht2013parallel, tanner2016low} abound. They typically assume that the linear measurement operators satisfy certain properties, such as the restricted isometry property (RIP) \cite{fazel2008compressed, mohan2010new} or restricted strong convexity \cite{negahban2012restricted}, propose reconstruction algorithms, and obtain reconstruction error guarantees on $X$. 

In many practical applications, the observation model deviates from the standard compressed sensing framework. Nonlinear measurement operators or structured observation patterns (that may not satisfy the RIP) are common. Examples include affine measurements \cite{fazel2008compressed, zuk2015low, cai2015rop} and quantized linear measurements \cite{jacques2013robust, davenport20141,  goldstein2018structured}. In some cases, nonlinear measurements can be reformulated as linear measurements with noise, using techniques such as generalized Lasso \cite{plan2016generalized, plan2017high, thrampoulidis2015lasso}. More often, a case-by-case study is needed. For example, \cite{papadimitriou2021data} approximates the ReLU function $\rho$ by a linear projection $P_{\Omega}$, where $\Omega$ indexes the positive entries of $\rho({X}_0)$.



Among these contributions, our proofs of \cref{thm:short two} and \cref{thm:short nonl} adapts methods from \cite{davenport20141}, which investigates one-bit (sign) observations of linear measurements.

\section{Theoretical Guarantees Under a Strong Low Rank Assumption}\label{sec:strong thm}
We make our first simplifying assumption for a pretrained multi-layer perceptron, which applies to all the pre-activations. To be specific, we assume there exist low rank matrices $M^{(i)}$ with rank $r_i$ ($i\geq 1$) such that 
 the neural network,  $\widecheck{\mathbf{\Phi}}$, with the same architecture as $\mathbf{\Phi}$ but weight matrices $M^{(i)}$ instead of $W^{(i)}$
satisfies: 
\begin{equation}\label{assumption one}
    \|{\mathbf{\Phi}}^{(i)}(X)W^{(i+1)}-(\widecheck{\mathbf{\Phi}}^{(i)}(X) M^{(i+1)}+G^{(i+1)})\|_{op}^2\leq \epsilon_{i+1} m  \ , i\geq 0
\end{equation}
 where $G^{(i)}$ are zero-mean sub-Gaussian matrices with i.i.d entries of variance $\sigma^2_i$ and $\epsilon_i$ are small tolerances. 
 

The following theorem, applicable to any layer, demonstrates that under our model assumptions, the underlying weight matrix can be easily approximated by solving a Frobenius norm minimization problem. Consequently, to simplify notation, we drop the layer index $i$ and simply denote by $X$ the input activation from the previous layer in the pretrained model $\mathbf{\Phi}$, and  by $\widecheck{X}$ the corresponding input activation for the low-rank model $\widecheck{\mathbf{\Phi}}$.

\begin{theorem}\label{thm:recovery one}(First Recovery Theorem)~\\
    Let $X, \widecheck{X} \in \mathbb{R}^{d_1 \times d}$, $d_1\geq d$ be full rank. 
    Let $W\in \mathbb{R}^{d \times d_2}$ be the  weight matrix from the pretrained model. Assume that there exists a rank-$r$ matrix $M\in \mathbb{R}^{d \times d_2}$ such that $\|XW - (\widecheck{X}M+G)\|_{op}^2\leq \epsilon d_1$, where $G\in \mathbb{R}^{d_1 \times d_2}$ is a zero-mean sub-Gaussian matrix with i.i.d entries of variance $\sigma^2$. Then, 
    $\hat{M}:=\underset{\mathrm{rank}(Z)\leq r}{\argmin} \|XW-\widecheck{X}Z\|_F$ 
    satisfies $$\frac{\|\widecheck{X}M-\widecheck{X}\hat{M}\|_F^2}{d_1 d_2}\lesssim r\cdot \frac{d_2+d}{d_1 d_2}\sigma^2+\epsilon$$ with probability at least $1-2e^{-(d_2+d)}$. 
    \begin{proof}
        Let $Y=\widecheck{X}M$ and $\tilde{Y}=XW=Y+G+E$ where $\|E\|_{op}^2\leq \epsilon d_1$. Observe that
        $$
        \|\tilde{Y}-\widecheck{X}\hat{M}\|_F^2 = \|\mathcal{P}_{\widecheck{X}}\tilde{Y}-\widecheck{X}\hat{M}\|_F^2+\|\mathcal{P}_{\widecheck{X}^\perp}\tilde{Y}\|_F^2.
        $$
        Here $\mathcal{P}_{\widecheck{X}}=\widecheck{X}\widecheck{X}^\dagger$ is the projection onto the column span of $\widecheck{X}$ and $\mathcal{P}_{\widecheck{X}^\perp}=I-\widecheck{X}\widecheck{X}^\dagger$ is its orthogonal complement. The second term is constant. For the first term, we have:
        $$
        \|\mathcal{P}_{\widecheck{X}}\tilde{Y}-\widecheck{X}\hat{M}\|_F=\|\widecheck{X}\widecheck{X}^\dagger\tilde{Y}-\widecheck{X}\hat{M}\|_F=\|(\widecheck{X}^\top\widecheck{X})^{1/2}\widecheck{X}^\dagger\tilde{Y}-(\widecheck{X}^\top\widecheck{X})^{1/2}\hat{M}\|_F
        $$
        Since $\mathrm{rank}((\widecheck{X}^\top\widecheck{X})^{1/2}\hat{M})=\mathrm{rank}(\hat{M})=r$, the optimal $\hat{M}$ achieves $(\widecheck{X}^\top\widecheck{X})^{1/2}\hat{M}=[(\widecheck{X}^\top\widecheck{X})^{1/2}\widecheck{X}^\dagger\tilde{Y}]_r$, where $[A]_r$ represents the best rank-$r$ approximation of $A$. This gives the explicit formula for the optimizer $\hat{M}=(\widecheck{X}^\top\widecheck{X})^{-1/2}[(\widecheck{X}^\top\widecheck{X})^{1/2}\widecheck{X}^\dagger\tilde{Y}]_r$. Then we have
        \begin{align*}
            \|\widecheck{X}M-\widecheck{X}\hat{M}\|_F 
            =&\|(\widecheck{X}^\top\widecheck{X})^{1/2}M-(\widecheck{X}^\top\widecheck{X})^{1/2}\hat{M}\|_F  \\
            =&\|(\widecheck{X}^\top\widecheck{X})^{1/2}M-[(\widecheck{X}^\top\widecheck{X})^{1/2}\widecheck{X}^\dagger\tilde{Y}]_r\|_F \\
            =&\|(\widecheck{X}^\top\widecheck{X})^{1/2}M-[(\widecheck{X}^\top\widecheck{X})^{1/2}\widecheck{X}^\dagger(\widecheck{X}M+G+E)]_r\|_F.
        \end{align*}
        Let $Z=(\widecheck{X}^\top\widecheck{X})^{1/2}M$ and $\tilde{Z}=(\widecheck{X}^\top\widecheck{X})^{1/2}\widecheck{X}^\dagger(\widecheck{X}M+G+E)=Z+\tilde{G}+\tilde{E}$, where $\tilde{G}=(\widecheck{X}^\top\widecheck{X})^{1/2}\widecheck{X}^\dagger G$ and similarly $\tilde{E}=(\widecheck{X}^\top\widecheck{X})^{1/2}\widecheck{X}^\dagger E$. Since both $Z$ and $\tilde{Z}_r$ are of rank $r$, $Z-\tilde{Z}_r$ has rank at most $2r$. Then we have:
        \begin{align*}
            \|\widecheck{X}M-\widecheck{X}\hat{M}\|_F=
            &\|Z-\tilde{Z}_r\|_F \\
            \leq&\sqrt{2r}\|Z-\tilde{Z}_r\|_2 \\
            \leq&\sqrt{2r}(\|Z-\tilde{Z}\|_2+\|\tilde{Z}-\tilde{Z}_r\|_2) \\
            \leq&\sqrt{2r}(\|\tilde{G}+\tilde{E}\|_2+\|\tilde{G}+\tilde{E}\|_2)\\
            =&2\sqrt{2r}\|\tilde{G}+\tilde{E}\|_2
        \end{align*}
        Here in the last inequality we used Weyl's Theorem \cite{stewart1998perturbation}:
        $$
        \|\tilde{Z}-\tilde{Z}_r\|_2=\sigma_{r+1}(\tilde{Z})\leq \sigma_{r+1}(Z)+\|\tilde{G}+\tilde{E}\|_2 = \|\tilde{G}+\tilde{E}\|_2.
        $$
        Now it remains to control $\|\tilde{G}\|_2$ and $\|\tilde{E}\|_2$. Let $\widecheck{X}=U\Sigma V^\top$ be the compact SVD of $\widecheck{X}$, then $\tilde{G}=VU^\top G$. We know $\|\tilde{G}\|_2=\|\tilde{G}^\top\|_2$ and $\tilde{G}^\top \in \mathbb{R}^{d_2 \times d}$ is still a Gaussian matrix with independent mean-zero isotropic rows, which satisfies $\|\tilde{G}^\top\|_2\lesssim_{\sigma} \sqrt{d_2}+ CK^2 (\sqrt{d}+t)$ with probability at least $1-2\mathrm{exp}(-t^2)$ by, e.g., Theorem 4.6.1 in \cite{vershynin2020high}. 
        We may choose $t=\sqrt{d_2+d}$ so that $\|\tilde{G}^\top\|_2^2\leq C_{\sigma}(d_2+d)$, where $C_{\sigma}$ is quadratic in $\sigma$. We also have $\|\tilde{E}\|_2=\|VU^\top E\|_2\leq \|VU^\top\|_2 \|E\|_2 = \|E\|_2 \leq \sqrt{\epsilon d_1}$. Thus,
        $$
        \|\widecheck{X}M-\widecheck{X}\hat{M}\|_F^2\leq 16r(C_{\sigma}(d_2+d)+\epsilon d_1)
        \leq 16rC_{\sigma}(d_2+d)+16\epsilon d_1(d_2\wedge d),
        $$
        and we can control the mean square error by:
        $$
        \frac{\|\widecheck{X}M-\widecheck{X}\hat{M}\|_F^2}{d_1 d_2}
        \leq (A\sigma^2)r \frac{d_2+d}{d_1 d_2}+B\epsilon.
        $$
        for constants $A$ and $B$.
    \end{proof}
\end{theorem}

Now let us compare the above result with what one would obtain from replacing $W$ by its best rank-$r$ approximation, $W_r$. One immediate difficulty is that one cannot control the Frobenius norm, without further assumptions, as
\begin{align*}
    &\|\widecheck{X}M-\widecheck{X}W_r\|_F \\
    =&\|\widecheck{X}M-\widecheck{X}W_r+XW_r-XW_r\|_F \\
    \leq &\|(X-\widecheck{X})W_r+X(W-W_r)-G\|_F + \|XW - (\widecheck{X}M+G)\|_F\\
    \leq&\|(X-\widecheck{X})W_r+X(W-W_r)\|_F+\|G\|_F+\|XW - (\widecheck{X}M+G)\|_F
\end{align*}
Note that the noise term alone  gives $\|G\|_F\lesssim\sqrt{d_1d_2}$, thus making the mean square error 
estimate $\mathcal{O}(1)$. This means we do not get any error decay guarantee. If, instead, we control the  Frobenius norm by the operator norm as we did in the proof of \cref{thm:recovery one}, then
    \begin{align*}
        \|\widecheck{X}M-\widecheck{X}W_r\|_F 
        \leq&\sqrt{2r}\|\widecheck{X}M-\widecheck{X}W_r\|_2 \\
        \leq&\sqrt{2r}(\|XW-\widecheck{X}W_r-G\|_2+\|XW - (\widecheck{X}M+G)\|_2) \\
        =&\sqrt{2r}(\|X(W-W_r)+(X-\widecheck{X})W_r-G\|_2+\|XW - (\widecheck{X}M+G)\|_2) \\
        \leq&\sqrt{2r}(\|(X-\widecheck{X})W_r+X(W-W_r)\|_2+\|G\|_2+\epsilon d_1) 
    \end{align*}
Now, with $\|G\|_2\lesssim\sqrt{d_1+d_2}$, we achieve the desired order on that term. However, since $W$ is not necessarily low rank and $X\neq \widecheck{X}$, controlling the error remains challenging without additional information about $X,\widecheck{X}$ and the spectrum of $W$. 
This highlights a potential explanation for why approximating each layer's weight matrix independently, without considering the input data, leads to rapid error accumulation. Consequently, extensive retraining is often required to restore accuracy, underscoring the need for data-driven compression algorithms to better preserve model performance.



\section{Theoretical Guarantees Under a Weak Low Rank Assumption}\label{sec:weak theorem}

The assumptions of last section require the existence of a ground truth matrix $M^{(i+1)}$ that is exactly low-rank. As we have argued, one difficulty in explaining the effectiveness of low-rank compression lies in the fact that weight matrices from pretrained models are typically not exactly low-rank.  On the other hand, recent works on low-rank compression of language models \cite{yu2023compressing}, implicit bias \cite{arora2019implicit, huh2021low, timor2023implicit} and neural collapse \cite{papyan2020prevalence, seleznova2024neural} suggest that the ``features" at intermdiate layers, ${\mathbf{\Phi}}^{(i)}(X)$, are more likely to have a nearly low-rank structure than their corresponding weights $W^{(i+1)}$.
Thus, we now make a more realistic and weaker assumption for a pretrained multi-layer perceptron. This assumption will guide our design of the optimization problem we study as well as its corresponding theoretical guarantees. In particular, we assume there exist matrices $M^{(i)}$, not necessarily low rank, such that \begin{equation}\label{assumption two}
    \|{\mathbf{\Phi}}^{(i)}(X)W^{(i+1)}-(\widecheck{\mathbf{\Phi}}^{(i)}(X) M^{(i+1)}+G^{(i+1)})\|_{F}^2\leq \epsilon_{i+1} m N_{(i+1)} \ , i\geq 0
\end{equation}
where $G^{(i)}$ are zero-mean sub-Gaussian matrices with i.i.d entries of variance $\sigma^2_i$ and $\epsilon_i$ are small tolerances.
Additionally, we only assume the pre-activation $\widecheck{\mathbf{\Phi}}^{(i)}(X) M^{(i+1)}$ in $\widecheck{\mathbf{\Phi}}$ are approximately rank-$r_{i+1}$, a concept we now define.

\begin{definition}\label{def:approx lr}
We say a matrix $Y\in \mathbb{R}^{d_1 \times d_2}$ is approximately rank-$r$ if it satisfies $\|Y\|_*\leq \|Y\|_{\infty} \sqrt{r d_1 d_2}$.
\end{definition}

The following two remarks justify our claims that our assumption is weaker and more realistic than those of the previous section. As in the previous section, we drop the layer indices and use $X, W, \widecheck{X}, M$ for notational simplicity. 

\begin{remark}
Assuming that $\widecheck{X}M$ is approximately rank-$r$ is  weaker than assuming $M$ is rank-$r$ due to the fact that $\mathrm{rank}(\widecheck{X}M)\leq \mathrm{rank}(M)$. Specifically, if $\mathrm{rank}(M)=r$, we have $$\|\widecheck{X}M\|_*\leq \sqrt{r} \|\widecheck{X}M\|_F \leq \sqrt{r m N} \|\widecheck{X}M\|_{\infty}.$$ 
\end{remark}
\begin{remark}
    $\widecheck{X}M$ being approximately low-rank will also holds true if $\widecheck{X}$ is approximately low-rank, without assuming any low-rank property on $W$. Let $\widecheck{X}\in \mathbb{R}^{d' \times d}$, $d'>d$ be full rank and approximately low-rank $\|\widecheck{X}\|_*\leq \|\widecheck{X}\|_{\infty} \sqrt{r d'd}$. Let $M\in \mathbb{R}^{d \times d}$ be an arbitrary matrix. Then $\|\widecheck{X}M\|_*
    \leq 
\|M\|_{\mathrm{op}}\|\widecheck{X}\|_*$ by Hölder's inequality for Schatten norms. By our assumptions, $\|\widecheck{X}M\|_* \leq \|\widecheck{X}M\|_{\infty} \sqrt{r d'd}\cdot I(\widecheck{X},M)$, where $I(\widecheck{X},M)=\frac{\|\widecheck{X}\|_{\infty}\|M\|_{\mathrm{op}}}{\|\widecheck{X}M\|_{\infty}}$ measures how $\widecheck{X}$ interacts with $M$. If the magnitudes of $\widecheck{X}, \widecheck{X}M$ are similar and $\|M\|_{\mathrm{op}}$ is  bounded, then the index $I(\widecheck{X},M)$ is well controlled. Intuitively, we can expect $\widecheck{X}M$ to be more low-rank than $\widecheck{X}$ since $\mathrm{rank}(\widecheck{X}M)\leq \min\{\mathrm{rank}(\widecheck{X}), \mathrm{rank}(M)\}$.
\end{remark}

We now provide a corresponding recovery theorem where we assume that both the pre-activations associated with $M$ and the random noise are bounded. Such assumptions are reasonable because the pre-activations in real world models are typically bounded due to regularization techniques that penalize large weights. Furthermore, assuming a general sub-Gaussian distribution for $G$ implies that its entries will be bounded by $\sqrt{\log d_1 d_2}$ with high probability (see \cref{gaussian infty}), so the noise assumption in this theorem is in fact similar compared to \cref{thm:recovery one}. Nevertheless, we will address unbounded Gaussian noise in the subsequent nonlinear recovery theorem (\cref{thm:nonlinear recovery}) whose proof, which is more involved,  is in  \cref{sec:prove nonlinear theorem}. 


\begin{theorem}\label{thm:recovery two}(Second Recovery Theorem)~\\
    Let $\widecheck{X} \in \mathbb{R}^{d_1 \times d}$. Assume there exists a matrix $M\in \mathbb{R}^{d \times d_2}$ such that $\widecheck{X}M$ is approximately rank-$r$ and $\widetilde{Y} = \widecheck{X}M+G$, where $G\in \mathbb{R}^{d_1 \times d_2}$ has i.i.d entries of bounded zero-mean random variables. We assume $\|\widecheck{X}M\|_{\infty}\leq \alpha$ and $\|G\|_{\infty}\leq \beta$.
    Let $$\Omega:=\{N: N\in \mathbb{R}^{d \times d_2}, \|\widecheck{X}N\|_*\leq\alpha \sqrt{r d_1 d_2};\ \|\widecheck{X}N\|_{\infty}\leq \alpha \}.$$
    Then, minimizing the linear reconstruction $$\hat{M}\in\underset{Z\in \Omega}{\argmin} \|\widetilde{Y}-\widecheck{X}Z\|_F$$ ensures that the mean square error satisfies$$\frac{\|\widecheck{X}M-\widecheck{X}\hat{M}\|_F^2}{d_1 d_2}\lesssim_{\alpha, \beta} \sqrt{\frac{r(d_1+d_2)}{d_1 d_2}}$$ with probability at least $1-\frac{K}{d_1+d_2}$ for an absolute constant $K$.
\end{theorem}
\begin{proof}
    The proof adapts techniques from \cite{davenport20141}.
    We first note that $\Omega$ is convex and $M\in \Omega$. 
    If we define $\Psi(\widecheck{X}):=\{Y\in \mathbb{R}^{d_1 \times d_2}: \|Y\|_*\leq\alpha \sqrt{r d_1 d_2};\ \|Y\|_{\infty}\leq \alpha ;\ Y_i\in \mathrm{span}\{\mathrm{col}(\widecheck{X})\}, i=1,......, d_2  \}$, then $\Omega$'s convexity is a direct consequence of $\Psi(\widecheck{X})$'s convexity, and we have $\widecheck{X}\Omega = \Psi(\widecheck{X})$. If $d_1\geq d$ and $\widecheck{X}$ is full rank, then there is a one to one mapping between $\Omega$ and $\Psi(\widecheck{X})$. In general, this mapping is many to one. By making the change of variables $Y=\widecheck{X}M$ and $\widetilde{Y}=Y+G$, proving the theorem is equivalent to proving that $\hat{Y}:=\underset{Z\in \Psi(\widecheck{X})}{\argmin} \|\widetilde{Y}-Z\|_F$ satisfies $\frac{\|Y-\hat{Y}\|_F^2}{d_1 d_2}\lesssim_{\alpha, \beta} \sqrt{\frac{r(d_1+d_2)}{d_1 d_2}}$ with high probability. For any $Z\in \Psi(\widecheck{X})$, let $\mathcal{L}(Z|\widetilde{Y})=\|\widetilde{Y}-Z\|_F^2=\underset{(i,j)}{\sum}(Z_{ij}-\widetilde{Y}_{ij})^2$. Center $\mathcal{L}(Z|\widetilde{Y})$ by setting $\widebar{\mathcal{L}}(Z|\widetilde{Y})=\mathcal{L}(Z|\widetilde{Y})-\mathcal{L}(\mathbf{0}|\widetilde{Y})=\underset{(i,j)}{\sum}(Z_{ij}^2-2\widetilde{Y}_{ij}Z_{ij})$. 
    
    The proof consists of two parts: bounding the deviation of $\widebar{\mathcal{L}}(Z|\widetilde{Y})$ from its mean and estimating $\mathbb{E}[\widebar{\mathcal{L}}(Y|\widetilde{Y})-\widebar{\mathcal{L}}(Z|\widetilde{Y})]$ for $Z\in \Psi(\widecheck{X})$.
    
    Since $\widetilde{Y}=Y+G$ and all the randomness is in $G$, we start with controlling the deviation of $\widebar{\mathcal{L}}(Z|\widetilde{Y})$ from its mean.     For any positive integer $h > 0$ and a constant $L_{\alpha,\beta}$ to be determined later, by Markov’s inequality we have that
    \begin{align*}        &\mathbb{P}\left(\underset{Z\in \Psi(\widecheck{X})}{\sup}|\widebar{\mathcal{L}}(Z|\widetilde{Y})-\mathbb{E}[\widebar{\mathcal{L}}(Z|\widetilde{Y})]|\geq C L_{\alpha,\beta} \alpha \sqrt{r d_1 d_2(d_1+d_2)}\right) \\
        &  \leq \mathbb{E}\left( \frac{\underset{Z\in \Psi(\widecheck{X})}{\sup}|\widebar{\mathcal{L}}(Z|\widetilde{Y})-\mathbb{E}[\widebar{\mathcal{L}}(Z|\widetilde{Y})]|^h}{(C L_{\alpha,\beta} \alpha \sqrt{r d_1 d_2(d_1+d_2)})^h}\right).
    \end{align*}
    Symmetrizing via Lemma \ref{symmetrization} yields
    $$
    \mathbb{E}[\underset{Z\in \Psi(\widecheck{X})}{\sup}|\widebar{\mathcal{L}}(Z|\widetilde{Y})-\mathbb{E}[\widebar{\mathcal{L}}(Z|\widetilde{Y})]|^h]\leq
    2^h \mathbb{E}[\underset{Z\in \Psi(\widecheck{X})}{\sup}|\underset{(i,j)}{\sum}\epsilon_{ij}(Z_{ij}^2-2\widetilde{Y}_{ij}Z_{ij})|^h],
    $$
    where the expectation on the left is over $Z$ (equivalently $G$) and the expectation on the right is over $Z$ and $\epsilon$. Here, all $\epsilon_{ij}$ are Rademacher random variables are independent of $Z$.

    To control the right hand side, we can apply the contraction principle \ref{contraction}. The function $z^2-2az$ defined on $[-\alpha,\alpha]$ is Lipschitz with constant less than $2(\alpha+|a|)$ and attains 0 when $z=0$. $\widetilde{Y}_{ij}=Y_{ij}+G_{ij}$ can be uniformly bounded by $\alpha+\beta$. Let $L_{\alpha,\beta} =4\alpha+2\beta$. Then the functions $\frac{1}{L_{\alpha,\beta}} (z^2-2\widetilde{Y}_{ij}z)$ are contractions. Defining the matrix $E$ with entries $\epsilon_{i,j}$, the contraction principle \ref{contraction} yields 
    \begin{align*}
        &\mathbb{E}[\underset{Z\in \Psi(\widecheck{X})}{\sup}|\widebar{\mathcal{L}}(Z|\widetilde{Y})-\mathbb{E}[\widebar{\mathcal{L}}(Z|\widetilde{Y})]|^h] \\
        \leq & 2^h (2L_{\alpha,\beta})^h \mathbb{E}[\underset{Z\in \Psi(\widecheck{X})}{\sup}|\underset{(i,j)}{\sum}\epsilon_{ij}Z_{ij}|^h] \\
        = & (4L_{\alpha,\beta})^h \mathbb{E}[\underset{Z\in \Psi(\widecheck{X})}{\sup}|\langle E, Z\rangle|^h] \\
        \leq & (4L_{\alpha,\beta})^h \mathbb{E}[\underset{Z\in \Psi(\widecheck{X})}{\sup}(\|E\|\|Z\|_*))^h] \\
        \leq & (4L_{\alpha,\beta})^h (\alpha \sqrt{r d_1 d_2})^h K (\sqrt{2(d_1+d_2)})^h.
    \end{align*}
    In the last inequality, we used the nuclear norm assumption on the space $\Psi$ and Lemma \ref{Rade norm}. Putting everything together,
    
    \begin{align*}
        &\mathbb{P}\left(\underset{Z\in \Psi(\widecheck{X})}{\sup}|\widebar{\mathcal{L}}(Z|\widetilde{Y})-\mathbb{E}[\widebar{\mathcal{L}}(Z|\widetilde{Y})]|\geq C L_{\alpha,\beta} \alpha \sqrt{r d_1 d_2(d_1+d_2)}\right) \\
        \leq & \frac{K (4L_{\alpha,\beta}\alpha \sqrt{r d_1 d_2(d_1+d_2)})^h}{(C L_{\alpha,\beta}\alpha \sqrt{r d_1 d_2(d_1+d_2)})^h}.
    \end{align*}
    With the choice $h\geq \log(d_1+d_2)$, the  probability is bounded from above by $\frac{K}{d_1+d_2}$ provided $C\geq 4\sqrt{2}e$.

    To conclude the proof, first note that (the ground truth) $Y=\widecheck{X}M$  is in $\Psi(\widecheck{X})$. For any $Z \in \Psi(\widecheck{X})$, we have
    \begin{align*}
        \widebar{\mathcal{L}}(Y|\widetilde{Y})-\widebar{\mathcal{L}}(Z|\widetilde{Y})
        &= \mathbb{E}[\widebar{\mathcal{L}}(Y|\widetilde{Y})-\widebar{\mathcal{L}}(Z|\widetilde{Y})]+(\widebar{\mathcal{L}}(Y|\widetilde{Y})-\mathbb{E}[\widebar{\mathcal{L}}(Y|\widetilde{Y})])-(\widebar{\mathcal{L}}(Z|\widetilde{Y})-\mathbb{E}[\widebar{\mathcal{L}}(Z|\widetilde{Y})])\\
        &=\mathbb{E}[\widebar{\mathcal{L}}(Y|\widetilde{Y})-\widebar{\mathcal{L}}(Z|\widetilde{Y})]+2\underset{Z' \in \Psi(\widecheck{X})}{\sup}|\widebar{\mathcal{L}}(Z'|\widetilde{Y})-\mathbb{E}[\widebar{\mathcal{L}}(Z'|\widetilde{Y})]|.
    \end{align*}

    Since $\mathbb{E}[\widebar{\mathcal{L}}(Z|\widetilde{Y})]=\mathbb{E}[\underset{(i,j)}{\sum}(Z_{ij}^2-2(Y_{ij}+G_{ij})Z_{ij})]=\underset{(i,j)}{\sum}(Z_{ij}^2-2Y_{ij}Z_{ij})$, we can compute $\mathbb{E}[\widebar{\mathcal{L}}(Y|\widetilde{Y})-\widebar{\mathcal{L}}(Z|\widetilde{Y})]=\underset{(i,j)}{\sum}-(Y_{ij}-Z_{ij})^2$.
    Thus, we get $\underset{(i,j)}{\sum}(Y_{ij}-Z_{ij})^2+\widebar{\mathcal{L}}(Y|\widetilde{Y})-\widebar{\mathcal{L}}(Z|\widetilde{Y})\leq 2\underset{Z' \in \Psi(\widecheck{X})}{\sup}|\widebar{\mathcal{L}}(Z'|\widetilde{Y})-\mathbb{E}[\widebar{\mathcal{L}}(Z'|\widetilde{Y})]|$. Now plug in the minimizer $Z=\Hat{Y}$ to both sides and use $\widebar{\mathcal{L}}(Y|\widetilde{Y})\geq \widebar{\mathcal{L}}(\hat{Y}|\widetilde{Y})$ to get
    $$
    \|Y-\hat{Y}\|_F^2=\underset{(i,j)}{\sum}(Y_{ij}-\hat{Y}_{ij})^2\leq 2\underset{Z' \in \Psi(\widecheck{X})}{\sup}|\widebar{\mathcal{L}}(Z'|\widetilde{Y})-\mathbb{E}[\widebar{\mathcal{L}}(Z'|\widetilde{Y})]|\leq 2\alpha C L_{\alpha,\beta}\sqrt{r d_1 d_2(d_1+d_2)},
    $$
    where the last inequality holds with probability at least $1-\frac{K}{d_1+d_2}$. Lastly, dividing both sides with $d_1d_2$ concludes the proof.
\end{proof}

With the above theorem in hand, the standard lemmas below will allow us to prove a result for neural networks with our assumptions on the rank of the pre-activations.

\begin{lemma}
    Let $\mathcal{C}$ be a compact convex set in $\mathbb{R}^D$ and $\mathcal{P}$ be the projection operator onto $\mathcal{C}$. Then $\forall x \in \mathbb{R}^D$, $\mathcal{P}(x)$ is uniquely determined.
\end{lemma}

\begin{lemma}
    Let $\mathcal{C}$ be a compact convex set in $\mathbb{R}^D$ and $\mathcal{P}$ be the projection operator onto $\mathcal{C}$. Then we have $\mathcal{P}$ is a contraction, i.e. $\|\mathcal{P}(x)-\mathcal{P}(y)\|_2\leq \|x-y\|_2$, $\forall x,y \in \mathbb{R}^D$.
\end{lemma}

With the above lemmas, we have the following straightforward corollary.

\begin{corollary}\label{cor:recovery two}
    Let $X\in \mathbb{R}^{d_1 \times d}$ and $W\in \mathbb{R}^{d \times d_2}$ represent pretrained activation and weights. Assume $\|XW - (\widecheck{X}M+G)\|_F^2\leq \epsilon d_1d_2$, where $\widecheck{X}, M, G$ are as in the previous theorem. Then,  $$\hat{M}\in\underset{Z\in \Omega}{\argmin} \|XW-\widecheck{X}Z\|_F$$ yields a mean square error satisfying $$\frac{\|\widecheck{X}M-\widecheck{X}\hat{M}\|_F^2}{d_1 d_2}\lesssim (\alpha^2+\alpha\beta) \sqrt{\frac{r(d_1+d_2)}{d_1 d_2}}+\epsilon$$ with high probability.
\end{corollary}
\begin{proof}
    Let $\widetilde{Y}=\widecheck{X}M+G$. By \cref{thm:recovery two},  the minimizer $M^* \in \underset{Z\in \Omega}{\argmin} \|\widetilde{Y}-\widecheck{X}Z\|_F$ satisfies $$\frac{\|\widecheck{X}M-\widecheck{X} M^*\|_F^2}{d_1 d_2}\lesssim_{\alpha, \beta} \sqrt{\frac{r(d_1+d_2)}{d_1 d_2}}$$ with high probability.
    Recall that we defined $\Psi(\widecheck{X})=\{Y\in \mathbb{R}^{d_1 \times d_2}: \|Y\|_*\leq\alpha \sqrt{r d_1 d_2};\ \|Y\|_{\infty}\leq \alpha ;\ Y_i\in \mathrm{span}\{\mathrm{col}(\widecheck{X})\}, i=1,......, d_2  \}$ and we have $\widecheck{X}\Omega = \Psi(\widecheck{X})$. Denote $\mathcal{P}$ the projection onto $\Psi(\widecheck{X})$. By the definition of our minimization problem, we have $\widecheck{X} M^*=\mathcal{P}(\widetilde{Y})$ and $\widecheck{X}\hat{M}=\mathcal{P}(XW)$. Since projection is contractive, we have $\|\widecheck{X} M^* - \widecheck{X}\hat{M}\|_F^2 \leq \|\widetilde{Y} - XW\|_F^2\leq \epsilon d_1d_2$. Thus
    \begin{align*}
        \frac{\|\widecheck{X}M-\widecheck{X}\hat{M}\|_F^2}{d_1 d_2}
       &=\frac{\|(\widecheck{X}M-\widecheck{X} M^*)+(\widecheck{X} M^* - \widecheck{X}\hat{M})\|_F^2}{d_1 d_2} \\
       &\leq \frac{2(\|\widecheck{X}M-\widecheck{X} M^*\|_F^2+\|\widecheck{X} M^* - \widecheck{X}\hat{M}\|_F^2)}{d_1 d_2}\\
       &\lesssim (\alpha^2+\alpha\beta) \sqrt{\frac{r(d_1+d_2)}{d_1 d_2}}+\epsilon.
    \end{align*}

\end{proof}

\begin{remark}[Further Remarks on the Approximately Low-Rank Constraint]
    Let $\Delta = \{Y : Y \in \mathbb{R}^{d_1 \times d_2}, \|Y\|_* \leq \alpha \sqrt{r d_1 d_2};\ \|Y\|_{\infty} \leq \alpha \}$ and $S = \{Y : Y \in \mathbb{R}^{d_1 \times d_2}, \mathrm{rank}(Y) \leq r;\ \|Y\|_{\infty} \leq \alpha \}$. As discussed in \cref{def:approx lr}, the assumption of being approximately rank-$r$ is weaker than the strict requirement of $\mathrm{rank}(Y) = r$ and it follows that $\mathrm{conv}(S) \subseteq \Delta$, where $\mathrm{conv}(\cdot)$ denotes the convex hull.  Precisely quantifying the difference between $\Delta$ and $\mathrm{conv}(S)$ is challenging, especially given the use of the $\ell_\infty$ norm in this setting. The $\ell_\infty$ norm is particularly relevant for neural networks, where various training and regularization techniques are employed to control activations and avoid instability. However, if the $\ell_\infty$ condition is replaced by the operator norm, classical results from the literature \cite{fazel2002matrix}, are typically applicable. Indeed, $\{\mathbf{X}:\|\mathbf{X}\|_*\leq1\}$ is the convex hull of the set of rank-1 matrices obeying $\|\mathbf{u}\mathbf{v}^\top\|_{\mathrm{op}}\leq 1$ but it is not clear whether an analogous result holds in our case. 
\end{remark} 

\if{
\begin{lemma}(Tightness)\\
    The set $\{\mathbf{X}:\|\mathbf{X}\|_*\leq1\}$ is the convex hull of the set of rank-1 matrices obeying $\|\mathbf{u}\mathbf{v}^\top\|_{\mathrm{op}}\leq 1$.
\end{lemma}
Let $\mathcal{C}$ be a given convex set. The \emph{convex envelope} of a (possibly non-convex) function $f:\mathcal{C}\rightarrow\mathbb{R}$ is defined to be the largest convex function $g$ which is bounded above by $f$ uniformly over $\mathcal{C}$.
\begin{theorem}
    The convex envelope of function $\mathrm{rank}(\cdot)$ on the set $\mathcal{C}=\{\|\mathbf{X}\|_{\mathrm{op}}\leq1\}$ is the nuclear norm $\|\cdot\|_*$.
\end{theorem}
}\fi

\if{
\begin{remark}
    Assume $X$ has independent symmetric
    Bernoulli random entries (The proof works for general sub-Gaussian entries.) and $XW$ is approximately low-rank, i.e. $$\|XW\|_*\leq \|XW\|_{\infty} \sqrt{r d'd}.$$
    Then $\|W\|_*=\|W\|_{S_1}=\|X^\dagger(XW)\|_{S_1}\leq \|X^\dagger\|_{S_\infty}\|XW\|_{S_1}=\|X^\dagger\|_{\mathrm{op}}\|XW\|_*$ by Hölder's inequality for Schatten norms. We have $\|X^\dagger\|_{\mathrm{op}}=\frac{1}{\lambda_{\min}(X)}\lesssim \frac{1}{\sqrt{d'}}$ with high probability by Theorem 4.6.1 in \cite{vershynin2020high}. Recall that we use $\|\cdot\|_{\infty}$ for element-wise $\ell_\infty$-norm of a matrix.

    Additionally, $\|XW\|_{\infty} \lesssim \sqrt{d\log(d'd)}\|W\|_{\infty}$ with high probability by Hoeffding’s inequality and a union bound, where I bound the $\ell_2$-norm of each row of $W$ by $\sqrt{d}\|W\|_{\infty}$.
    Putting things together, we can get
    $$
    \|W\|_* \lesssim \frac{1}{\sqrt{d'}} \sqrt{r d'd} \sqrt{d\log(d'd)}\|W\|_{\infty}=\sqrt{rd^2\log(d'd)}\|W\|_{\infty}
    $$
    with high probability.
    The only difference from the definition of $W$ being approximately low-rank is on the order of $\mathcal{O}(\sqrt{\log(d'd)})$.
\end{remark}
We will have similar conclusion for $XW$ if we start with assuming $W$ is approximately low rank $\|W\|_*\leq \|W\|_{\infty} \sqrt{r d^2}$. Thus, we have showed that just as assuming $XW$ low-rank is equivalent to directly assuming $W$ to be low-rank, assuming $XW$ approximately low-rank is also closely related to assuming $W$ to be approximately low-rank. In view of this, we choose to make our assumption on $\widecheck{X}M$ in \cref{thm:recovery two} and \cref{cor:recovery two}. This is also in the spirit of emphasizing our observation that weight matrices in pretrained models, while not exactly low-rank, exhibit low-rank behavior when interacting with data distributions.
}\fi

\if{\subsection{A Discussion on Whole Model Error(TBD)}\label{sec:whole model}\RSnote{If we want to keep this, we should start with the conclusion/summary of the result, and then either prove it or explain how it is obtained. We should also comment in the beginning on the significance of the result and why it is diffucult to improve.}

We now show \cref{cor:recovery two} is compatible with our model assumptions \ref{assumption two} following the notations set up in \cref{sec:notation} and derive a whole model error when the neural network is well behaved. For simplicity we assume $N_i=d$, $\forall i\geq 0$ and the sample size $m=d$. We add an assumption that matrices $M^{(i)}$ in the underlying approximately low rank model have a uniformly bounded operator norm $K\geq \|M^{(i)}\|_2^2$. This is a reasonable assumption as weight matrices of shape $d\times d$ with Xavier or He initialization will have operator norm of $\mathcal{O}(1)$ with high probability.

Let us populate a batch data $X_0 \in \mathbb{R}^{d\times d}$ into the network. By assumption, $\|X_0W^{(1)}-(X_0M^{(1)}+G^{(1)})\|_F^2\leq \epsilon_1 d^2$. \cref{cor:recovery two} tells that $\frac{\|X_0M^{(1)}-X_0\hat{M}^{(1)}\|_F^2}{d^2}\leq C(\alpha_1,\beta_1) \sqrt{\frac{r_1}{m}}+2\epsilon_1:=e_1$, where $C(\alpha_1,\beta_1)$ is a constant depending on $\alpha_1,\beta_1$ and $\hat{M}^{(1)}:=\underset{Z\in \Omega}{\argmin} \|X^{(0)}W^{(1)}-\widetilde{X}^{(0)}Z\|_F$. By our notation set up, $X^{(1)}=\rho(X_0W^{(1)})$, $\widecheck{X}^{(1)}=\rho(X_0M^{(1)})$ and $\widetilde{X}^{(1)}=\rho(X_0\hat{M}^{(1)})$. Using the Lipschitz property of our activation functions, we have $\|\widetilde{X}^{(1)}-\widecheck{X}^{(1)}\|_F^2 \leq e_1 d^2$. Then, $\|\widetilde{X}^{(1)}M^{(2)}-\widecheck{X}^{(1)}M^{(2)}\|_F^2 \leq K e_1 d^2$ by our assumption of a uniform bound on the operator norm of $M^{(i)}$.

Let us continue inductively. Using assumption for the second layer, we have $\|X^{(1)}W^{(2)}-(\widecheck{X}^{(1)}M^{(2)}+G^{(2)})\|_F^2\leq \epsilon_2 d^2$. Since we only have access to $\widetilde{X}^{(1)}$ instead of $\widecheck{X}^{(1)}$, we use triangle inequality to get $\|X^{(1)}W^{(2)}-(\widetilde{X}^{(1)}M^{(2)}+G^{(2)})\|_F^2\leq 2(\epsilon_2+K e_1) d^2$. Then, $\hat{M}^{(2)}:=\underset{Z\in \Omega}{\argmin} \|X^{(1)}W^{(2)}-\widetilde{X}^{(1)}Z\|_F$ will satisfy $\frac{\|\widetilde{X}^{(1)}M^{(2)}-\widetilde{X}^{(1)}\hat{M}^{(2)}\|_F^2}{d^2}
\leq C(\alpha_2,\beta_2) \sqrt{\frac{r_2}{m}}+4(\epsilon_2+K e_1)$. 

Remember we always compare the compressed model $\widetilde{\mathbf{\Phi}}$ with the underlying assumption model $\widecheck{\mathbf{\Phi}}$. By another application of triangle inequality to $\widecheck{X}^{(1)}M^{(2)}-\widetilde{X}^{(1)}\hat{M}^{(2)}=(\widecheck{X}^{(1)}M^{(2)}-\widetilde{X}^{(1)}M^{(2)})+(\widetilde{X}^{(1)}M^{(2)}-\widetilde{X}^{(1)}\hat{M}^{(2)})$, we have $\frac{\|\widecheck{X}^{(1)}M^{(2)}-\widetilde{X}^{(1)}\hat{M}^{(2)}\|_F^2}{d^2}\leq 2\left( C(\alpha_2,\beta_2) \sqrt{\frac{r_2}{m}}+4(\epsilon_2+K e_1)\right)+2Ke_1=10Ke_1+2\left( C(\alpha_2,\beta_2) \sqrt{\frac{r_2}{m}}+4\epsilon_2\right)$. 

If we denote $\tilde{\epsilon}:=2\max_i ( C(\alpha_i,\beta_i)\sqrt{\frac{r_i}{d}}+4\epsilon_i)$, then we immediately have $e_1=C(\alpha_1,\beta_1) \sqrt{\frac{r_1}{d}}+2\epsilon_1 \leq \tilde{\epsilon}$. Let $\gamma:=10K$. If we use $e_l$ to denote the mean square error of $\widetilde{\mathbf{\Phi}}$ against $\widecheck{\mathbf{\Phi}}$ at layer index $l$, we have
\begin{equation}\label{whole model}
    e_{l+1}\leq \gamma e_l+\tilde{\epsilon},
\end{equation}
 i.e. the mean square error obeys the sub-linear update rule. 

Iteratively applying \cref{whole model}, we can get the last layer mean square error 
\begin{align*}
    &\frac{\|\widecheck{X}^{(L-1)}M^{(L)}-\widetilde{X}^{(L-1)}\hat{M}^{(L)}\|_F^2}{d^2} \\
    =&e_L \\
    \leq&\gamma e_{L-1}+\tilde{\epsilon} \\
    \leq&...... \\
    \leq&\gamma^{L-1}e_1+\tilde{\epsilon}(\gamma^{L-2}+......+\gamma+1)\\
    \leq&\tilde{\epsilon}(\gamma^{L-1}+......+\gamma+1) \\
    \leq&\tilde{\epsilon}\cdot \begin{cases}
    \frac{\gamma^L-1}{\gamma-1}, & \text{if } \gamma>1 \\
    \frac{1}{1-\gamma}, & \text{if} \gamma<1
    \end{cases}
\end{align*}
From the last inequality, we can also see why modern training favors regularization techniques like weight decay, spectral normalization or orthogonal regularization to control entry magnitudes and operator norms. If $\gamma<1$, then the error will remain nicely controlled. When $\gamma>1$, the error increases exponentially with the depth of the network. 

Assume the parameters $\alpha_i, \beta_i, r_i$ remains bounded and we over-parameterize the neural network, i.e. let $d$ be large. When the pre-activations in the original neural network is almost a noisy version of those from an underlying approximately low rank network, which means the $\epsilon_i$ are all tiny, then $\tilde{\epsilon}$ is on the order of $\mathcal{O}(\frac{1}{\sqrt{d}})$. }\fi

\subsection{Non-linear Recovery Theorem}

\Cref{thm:recovery two} demonstrates that by minimizing the linear reconstruction error (i.e., the error in approximating the pre-activations) one can approximately recover $M$ with high probability from $ XW \approx \widecheck{X}M +G$. In the context of neural network compression via low-rank approximation, prior works \cite{zhang2015accelerating, papadimitriou2021data} have also explored recovering $M$ from the (non-linear) activations $ \rho(XW) \approx \rho(\widecheck{X}M +G)$. This often entails solving $\underset{N}{\min} \|\rho(XW)-\rho(\widecheck{X}N)\|_F$ instead of its linearized counterpart $\underset{N}{\min} \|XW-\widecheck{X}N\|_F$. Empirical results in, e.g.,  \cite{zhang2015accelerating} suggest that accounting for the non-linearity yields better low-rank compression before fine-tuning. 

Deriving theory to explain this observation is non-trivial for a number of reasons. On the one hand, as neural network loss functions typically depend on the activations $\rho(XW)$, then approximating this quantity should in principle yield better results. On the other hand, the approximation task itself is more difficult for at least two reasons. First, it involves the added challenge of dealing with the non-convexity and non-smoothness introduced by $\rho$. Second, from a signal recovery perspective, recovering $M$ from the non-linear observations $\rho(\widecheck{X}M+G)$ is more difficult since $\rho$ sets all negative values to zero, thereby  eliminating information. 


The following theorem establishes that a comparable error bound--up to constants and logarithmic factors--holds when recovering $M$ from $Z = \rho(\widecheck{X}M+G)$ via minimizing a tight convex relaxation of $\|Z-\rho(\widecheck{X}N)\|_F$. A more detailed discussion of this relaxation is provided in \cref{sec:con to frob}. The proof is technically more intricate than that of \Cref{thm:recovery two}. Moreover, an additional $\sqrt{\log d}$ term in the error bound accounts for potential outliers in $Z$ caused by the unbounded noise $G$.

\begin{theorem}\label{thm:nonlinear recovery}(Nonlinear Recovery Theorem)~\\
    Let $\widecheck{X}  
    \in \mathbb{R}^{d_1 \times d}$, $d_1\geq d$ be a full rank matrix with $\check{x}_i^\top, i=1,...,d_1$ as its rows, and let $M\in \mathbb{R}^{d \times d_2}$ 
    be such that $\widecheck{X}M$ is approximately rank-$r$ with $\|\widecheck{X}M\|_{\infty}\leq \alpha$. Let $Z = \rho(\widecheck{X}M+G)$, where $G\in \mathbb{R}^{d_1 \times d_2}$ is a random Gaussian matrix with i.i.d. $\mathcal{N}(0, \sigma^2)$ entries.  Define $\Omega=\{N: N\in \mathbb{R}^{d \times d_2}, \|\widecheck{X}N\|_*\leq\alpha \sqrt{r d_1 d_2};\ \|\widecheck{X}N\|_{\infty}\leq \alpha \}$ and    denote by $f(x)$ the CDF of the normal distribution $\mathcal{N}(0, \sigma^2)$, then with probability at least $1-(\frac{K}{d_1+d_2}+\frac{1}{2\sqrt{2\pi}}\frac{1}{d_1d_2\sqrt{\log(d_1d_2)}})$, the solution $\hat{M}$ to 
    \begin{align*}\label{opt P star prime}
    &\max_N\underset{(i,j):Z_{ij}>0}{\sum}\log\left(\frac{1}{\sqrt{2\pi}\sigma}e^{-\frac{(Z_{ij}-\langle \check{x}_i,N_j \rangle)^2}{2\sigma^2}}\right)+\underset{(i,j):Z_{ij}=0}{\sum}\log\left(1-f(\langle \check{x}_i,N_j \rangle)\right) \tag{$P_*'$} \\
    &\text{subject to} \ N \in \Omega \  
    \end{align*}
    satisfies
    \begin{equation}\label{result}
    \frac{1}{d_1d_2}\|\widecheck{X}M-\widecheck{X}\hat{M}\|_F^2 \leq C_{\alpha, \sigma} \max\left\{2\sqrt{\log(d_1d_2)}, 8\right\} \sqrt{\frac{r(d_1+d_2)}{d_1 d_2}}.
    \end{equation}
    Here, $K$ is an absolute constant. $C_{\alpha, \sigma}=16C\alpha\beta_{\alpha, \sigma}\gamma_{\alpha, \sigma}$ where $C$ is an absolute constant, $\beta_{\alpha, \sigma} = \pi \sigma^2 e^{\alpha^2/2\sigma^2}$ and $\gamma_{\alpha, \sigma} =\frac{\alpha + \sigma}{\sigma^2}$. $N_j$ is the $j$-th column of $N$.
\end{theorem}


The proof follows a similar strategy to that of \cref{thm:recovery two}, but is more technically involved due to the nonlinearity and the unbounded nature of the noise. The complete proof is in \cref{sec:prove nonlinear theorem}, which starts with a reduction from (\ref{opt P star prime}) to 
(\ref{opt P star}):
\begin{align}\label{opt P star}
   &\max_{M'}{\sum\limits_{(i,j):Z_{ij}>0}}\log\left(\frac{1}{\sqrt{2\pi}\sigma}e^{-\frac{(Z_{ij}-M'_{ij})^2}{2\sigma^2}}\right)+{\sum\limits_{(i,j):Z_{ij}=0}}\log(1-f(M'_{ij})) \tag{$P_*$} 
   &\text{subject to} \ M' \in \Psi(\widecheck{X})   
\end{align}

\begin{remark}
    It is easy to observe that the objective function in (\ref{opt P star}) is concave. 
    The objective function of (\ref{opt P star prime}) is still concave by the lemma below.
\end{remark}

\begin{lemma}\label{convexity after change of variable}
    Let $X 
\in \mathbb{R}^{m \times d}$ with $m>d$ be a full column rank matrix, with $x_i^\top, i=1,...,d_1$ as its rows. Let $h_i, i\in[m]: \mathbb{R} \rightarrow \mathbb{R}$ be strictly convex (concave) functions. Then $H(w)=\sum_{i=1}^m h_i(x_i^\top w), \mathbb{R}^d \rightarrow \mathbb{R}$ is a strictly convex (concave) function. 
    \begin{proof}
        As the proof for concave and convex functions is essentially identical, we only provide it for convex functions here. For any fixed $w\in\mathbb{R}^d$, the second  derivatives $h_i''(x_i^\top w), i=1,......,m$ are all positive by strict convexity. Let $C=\min_i h_i''(x_i^\top w)>0$. By a direct calculation, and using the fact that $X$ is full column rank, 
        \begin{align*}
            \nabla^2H(w)&=\sum_{i=1}^m \nabla^2h_i(x_i^\top w)
            =\sum_{i=1}^m h_i''(x_i^\top w) x_i x_i^\top \\
            &\succcurlyeq C \sum_{i=1}^m x_i x_i^\top 
            = C X^\top X 
            \succ 0.
        \end{align*}
         Thus $H: \mathbb{R}^d \rightarrow \mathbb{R}$ is strictly convex. 
    \end{proof}
\end{lemma}

\begin{remark}
    The optimization problem (\ref{opt P star}) can be interpreted as a maximum likelihood estimation problem. In our context, the log-likelihood loss function is given by 
    \[
    \mathcal{L}(M'|Z) = \sum_{(i,j)} \log L(M'_{ij}|Z_{ij}) = \sum_{(i,j):Z_{ij}>0} \log L(M'_{ij}|Z_{ij}) + \sum_{(i,j):Z_{ij}=0} \log L(M'_{ij}|Z_{ij}),
    \]
    where the likelihood \(L(M'_{ij}|Z_{ij})\) depends on whether \(Z_{ij}\) is positive or zero. When \(Z_{ij} > 0\), the likelihood is a continuous density given by \(L(M'_{ij}|Z_{ij}) = f'(Z_{ij} - M'_{ij})\). When \(Z_{ij} = 0\), the likelihood becomes discrete, with \(L(M'_{ij}|Z_{ij}) = \mathbb{P}(G_{ij} + M'_{ij} \leq 0) = f(-M'_{ij}) = 1 - f(M'_{ij})\). Substituting the expressions for \(L(M'_{ij}|Z_{ij})\), with \(L(M'_{ij}|Z_{ij}) = \frac{1}{\sqrt{2\pi}\sigma} e^{-\frac{(Z_{ij} - M'_{ij})^2}{2\sigma^2}}\) when \(Z_{ij} > 0\), and \(L(M'_{ij}|Z_{ij}) = 1 - f(M'_{ij})\) when \(Z_{ij} = 0\), we recover (\ref{opt P star}).
\end{remark}

\if{\begin{remark}
    An amazing fact about the optimization problem (\ref{opt P star}) is that it can be interpreted as finding a maximum likelihood estimator. The log-likelihood loss function is given by $\mathcal{L}(M'|Z)=\underset{(i,j)}{\sum}\log L(M'_{ij}|Z_{ij})=\underset{(i,j):Z_{ij}>0} {\sum}\log L(M'_{ij}|Z_{ij})+\underset{(i,j):Z_{ij}=0}{\sum}\log L(M'_{ij}|Z_{ij})$. When $Z_{ij}>0$, the likelihood is a continuous density $L(M'_{ij}|Z_{ij})=f'(Z_{ij}-M'_{ij})$
    . When $Z_{ij}=0$, the likelihood is a discrete probability $L(M'_{ij}|Z_{ij})=\mathbbm{P}(G_{ij}+M'_{ij}\leq 0)=f(-M'_{ij})=1-f(M'_{ij})$. 

    In the following log-likelihood maximization problem,
    \begin{align*}
    &\underset{M'}{\text{maximize}} \ \mathcal{L}(M'|Z)=\underset{(i,j):Z_{ij}>0} {\sum}\log L(M'_{ij}|Z_{ij})+\underset{(i,j):Z_{ij}=0}{\sum}\log L(M'_{ij}|Z_{ij}), \\
    &\text{subject to} \ M' \in \Psi(\widecheck{X}). 
    \end{align*}
    Plugging in $L(M'_{ij}|Z_{ij})=\frac{1}{\sqrt{2\pi}\sigma}e^{-\frac{(Z_{ij}-M'_{ij})^2}{2\sigma^2}}$ when $Z_{ij}>0$ and $L(M'_{ij}|Z_{ij})=1-f(M'_{ij})$ when $Z_{ij}=0$, we immediately recover our convex optimization problem (\ref{opt P star}), which we restated below.
\begin{align*}
   &\underset{M}{\text{maximize}} \ \underset{(i,j):Z_{ij}>0}{\sum}\log(\frac{1}{\sqrt{2\pi}\sigma}e^{-\frac{(Z_{ij}-M'_{ij})^2}{2\sigma^2}})+\underset{(i,j):Z_{ij}=0}{\sum}\log(1-f(M'_{ij})), \tag{$P_*$} \\
   &\text{subject to} \ M' \in \Psi(\widecheck{X}). 
\end{align*}
\RSnote{Let's add a couple of sentences here stating how it is an ML estimator. The explanation is short after all.} \textcolor{blue}{[Moved the explantion from appendix to here, please feel free to shorten it if necessary.]}Furthermore, we show in \cref{sec:con to frob} that solving the optimization problem \ref{opt P star} is equivalent to minimizing a tight convex upper bound of $\frac{1}{2}\|Z - \rho(M')\|_F^2$, subject to $M' \in \Psi(\widecheck{X})$. This approach is analogous to the common practice of maximizing the evidence lower bound (ELBO) \cite{luo2022understanding}, which serves as a lower bound on the log-likelihood of observed data from an unknown distribution.
\end{remark}}\fi

\begin{remark}
    When $d_1 = d_2 = d$ is large, the right-hand side of inequality (\ref{result}) scales as $\mathcal{O}(\sqrt{\frac{\log d}{d}})$ if $\alpha$ and $\sigma$ are fixed, and $r$ remains bounded. This implies that the mean squared error still converges to $0$ as $d$ increases.
\end{remark}

\if{
\section{Discussions}\RSnote{I like this section, but I think it needs to go, for conciseness.}
\subsection{The challenge in recovery with noisy ReLU observation}
The challenge in recovery from ReLU activation lies in the loss of information for negative entries in the pre-activation. In fact for a low rank matrix $Y$, recovering $Y$ from $Z=\rho(Y)$ is in general not possible. For example, let 
$Y=C^TC$ be a $2m \times 2m $ matrix of rank 2, with $$C:= \left[\begin{array}{c;{2pt/2pt}c}
	    A  & B\\
	\end{array}\right] =\begin{bmatrix}
    |&......&|&|&......&| \\
    a_1&......&a_m&b_1&......&b_m \\
    |&......&|&|&......&|
\end{bmatrix},$$ where all $a_i$'s and $b_k$'s are column vectors in 
 $\mathbb{R}^2$, satisfying $\langle a_i,a_j \rangle>0$ for any $i,j \in [m]$, $\langle b_k,b_l \rangle>0$ for any $k,l \in [m]$, $\langle a_i,b_k \rangle<0$ for any $i,k \in [m]$. 
Then, $\rho(Y)=\left[\begin{array}{c;{2pt/2pt}c}
	A^\top A  & \mathbf{O}\\
     \hdashline
    \mathbf{O} & B^\top B
	\end{array}\right]$. Now let $P, Q \in \mathbb{R}^{2\times2}$ be two orthogonal matrices representing rotation. Let $\tilde{C}=\left[\begin{array}{c;{2pt/2pt}c}
	    PA  & QB\\
	\end{array}\right]$ and $\tilde{Y}=\tilde{C}^\top \tilde{C}=\left[\begin{array}{c;{2pt/2pt}c}
	A^\top A  & A^\top P^\top QB\\
     \hdashline
    B^\top Q^\top PA & B^\top B
	\end{array}\right]$. Then, there exist $P\neq Q$ and $P, Q$ are both small rotations, where we cannot distinguish between $Y\neq \tilde{Y}$ since $\rho(Y)=\rho(\Tilde{Y})$.

With the presence of noise, however, precise recovery becomes possible. Consider the simpler example where we have a scalar $y<0$, then observing $\rho(y)=0$ does not allow recovering $y$ itself. Conversely, consider noisy measurements with the (even more extreme non-linearity) given by the  sign function, i.e. $z=\sign(y+g)$, where $\text{sign}(x) = \begin{cases} 
      +1 & \text{if } x > 0\\
      -1 & \text{if } x \leq 0
\end{cases}$. 
Observing i.i.d. realizations of $z=\sign(y+g)$ where $g\sim \mathcal{N}(0, \sigma^2)$ will  yield a proportion of $-1$ observations that approaches 
$\mathbb{P}(g\leq -y)$ as $N\rightarrow \infty$ by the strong law of large numbers. Thus, we can get a good approximation of $y$ with enough observations by inverting the CDF of normal distribution. In our matrix recovery setting of observing $Z=\sign(Y+G)$ where $Y \in \mathbb{R}^{n \times n}$ is a low rank matrix with $rank(Y)=r$ and $r\ll n$.  $Y$ has $2nr-r^2=\mathcal{O}(nr)$ degrees of freedom and in this case we have $n^2$ measurements associated with $Y$ so there is hope to recover the low rank matrix $Y$ with high accuracy. \cref{thm:nonlinear recovery} we proved makes this intuition precise. 

}\fi

\section{Future Work}
This work opens several avenues for future research, and we now outline some of them.

    \emph{Role of Nonlinear Activation Functions.}  
    The nonlinear recovery theorem we proved does not explicitly account for how incorporating non-linear activation functions into the compression algorithm can mitigate accuracy loss compared to methods that ignore non-linearities. Developing a deeper theoretical understanding of the role played by non-linear activation functions in low-rank recovery remains an important direction for future research.

    \emph{Low-Rank Approximation for Higher-Order Tensors.}  Tensor decomposition techniques, such as Canonical Polyadic Decomposition (CPD) and Tucker Decomposition, are widely used for low-rank approximation of convolutional neural networks \cite{jaderberg2014speeding, nie2023low, liu2023tensor, price2023improved}. However, extending recovery theory from matrices to tensors poses challenges, as tensors lack a matrix-style SVD and an Eckart-Young theorem \cite{eckart1936approximation} (which states that the best rank-$k$ Frobenius and operator norm approximation of a matrix is obtained by truncating its SVD to the largest $k$ singular values). Recent advances in compressed sensing and statistical inference offer promising directions for establishing rigorous recovery guarantees for low-rank tensor decompositions, particularly for tensors with properties relevant to convolutional neural networks \cite{yuan2016tensor, pan2020low, xia2022inference, raskutti2019convex, auddy2023perturbation}. It would be interesting to investigate whether rigorous recovery theorems can be proved with the help of techniques from these works. 
    
    \emph{Gradient Descent-Based Algorithms for Low-Rank Recovery.}  
    Our current approaches rely on the SVD or on solving convex optimization problems but do not address specific algorithmic implementations. In practice, gradient descent (GD) and its variants are widely used for training neural networks. Recovery guarantees for GD-based algorithms, often tied to algorithmic regularization, have been explored in prior work \cite{du2018algorithmic, li2018algorithmic, jiang2023algorithmic}. Extending our guarantees to connect more explicitly with compression algorithms that resemble those to train neural networks is another promising research direction.

\if{\section{Future Work}
\RSnote{Might not need this (?)}\textcolor{blue}{[I commented out this paragraph and made some editions to this Future Work section.]}\RSnote{I meant the entire section :)}

As mentioned earlier, the nonlinear recovery theorem we proved does not capture the fact that explicitly account for the non-linearity in the compression algorithm can reduce the accuracy drop compared to algorithms that do not account for the non-linearity. Developing a better theoretical understanding of the role played by non-linear activation functions in the context of low-rank matrix recovery is an important direction for future research.

Many existing methods for low-rank approximation of convolutional neural networks \cite{jaderberg2014speeding, nie2023low, liu2023tensor, xiao2023haloc, ahmed2023speeding, price2023improved} are based on tensor decomposition techniques, such as Canonical Polyadic Decomposition (CPD) or Tucker Decomposition \cite{rabanser2017introduction, vandecappelle2017nonlinear, kolda2009tensor}. However, moving from matrices to higher-order tensors introduces several theoretical challenges. Specifically, higher-order tensors lack a matrix-style singular value decomposition (SVD) and an Eckart-Young theorem \cite{eckart1936approximation} \textcolor{blue}{which states that the best rank-$k$ approximation of a matrix (in terms of minimizing the Frobenius or spectral norm error) is obtained by truncating its SVD to the top $k$ singular values and corresponding singular vectors.} \RSnote{citation + brief explanation of what EY-theorem is (best rank-r...)}. The best low-rank approximation for higher-order tensors may not exist; that is, the optimal low-rank approximation problem can be ill-posed \cite{de2008tensor}. 

Recent progress in compressed sensing and statistical inference has extended recovery theory beyond sparse vectors and low-rank matrices to low-rank tensors \cite{ghadermarzy2018learning, wang2018sparse, xia2021statistically, yuan2016tensor, pan2020low, xia2022inference, yuan2017incoherent}. These advances offer promising directions to establish theoretical guarantees for tensor decomposition methods. We believe rigorous recovery theorems can be proved for tensors satisfying specific desired properties \cite{raskutti2019convex, auddy2023perturbation}, particularly in the context of convolutional neural networks. 

\RSnote{Other directions?}\textcolor{blue}{The proofs of our results rely on singular value decomposition (SVD) or solving convex optimization problems, without specifying the exact algorithm to use. In modern neural network training, gradient descent (GD)-based algorithms are the most widely used. Low-rank recovery guarantees are possible even when restricted to GD-based algorithms, as shown in works such as \cite{tanner2016low, du2018algorithmic, li2018algorithmic, jiang2023algorithmic} with different step-size choices, often referred to as algorithmic regularization. While these results are insightful, their connection to neural network architecture is relatively weak. A potential avenue for future work is to establish recovery results for GD within the context of neural networks.}

}\fi

\section*{Acknowledgment}
We gratefully acknowledge partial support by National Science Foundation, via the DMS-2410717 grant.

\bibliographystyle{abbrvnat}
\bibliography{citations}

\begin{thebibliography}{63}
\providecommand{\natexlab}[1]{#1}
\providecommand{\url}[1]{\texttt{#1}}
\expandafter\ifx\csname urlstyle\endcsname\relax
  \providecommand{\doi}[1]{doi: #1}\else
  \providecommand{\doi}{doi: \begingroup \urlstyle{rm}\Url}\fi

\bibitem[Arora et~al.(2019)Arora, Cohen, Hu, and Luo]{arora2019implicit}
S.~Arora, N.~Cohen, W.~Hu, and Y.~Luo.
\newblock Implicit regularization in deep matrix factorization.
\newblock \emph{Advances in Neural Information Processing Systems}, 32, 2019.

\bibitem[Auddy and Yuan(2023)]{auddy2023perturbation}
A.~Auddy and M.~Yuan.
\newblock Perturbation bounds for (nearly) orthogonally decomposable tensors with statistical applications.
\newblock \emph{Information and Inference: A Journal of the IMA}, 12\penalty0 (2):\penalty0 1044--1072, 2023.

\bibitem[Borzadaran and Borzadaran(2011)]{borzadaran2011log}
G.~M. Borzadaran and H.~M. Borzadaran.
\newblock Log-concavity property for some well-known distributions.
\newblock \emph{Surveys in Mathematics and its Applications}, 6:\penalty0 203--219, 2011.

\bibitem[Cai and Zhang(2015)]{cai2015rop}
T.~T. Cai and A.~Zhang.
\newblock Rop: Matrix recovery via rank-one projections.
\newblock 2015.

\bibitem[Chen et~al.(2021)Chen, Yu, Dhillon, and Hsieh]{chen2021drone}
P.~Chen, H.-F. Yu, I.~Dhillon, and C.-J. Hsieh.
\newblock Drone: Data-aware low-rank compression for large nlp models.
\newblock \emph{Advances in neural information processing systems}, 34:\penalty0 29321--29334, 2021.

\bibitem[Choudhary et~al.(2020)Choudhary, Mishra, Goswami, and Sarangapani]{choudhary2020comprehensive}
T.~Choudhary, V.~Mishra, A.~Goswami, and J.~Sarangapani.
\newblock A comprehensive survey on model compression and acceleration.
\newblock \emph{Artificial Intelligence Review}, 53:\penalty0 5113--5155, 2020.

\bibitem[Davenport and Romberg(2016)]{davenport2016overview}
M.~A. Davenport and J.~Romberg.
\newblock An overview of low-rank matrix recovery from incomplete observations.
\newblock \emph{IEEE Journal of Selected Topics in Signal Processing}, 10\penalty0 (4):\penalty0 608--622, 2016.

\bibitem[Davenport et~al.(2014)Davenport, Plan, Van Den~Berg, and Wootters]{davenport20141}
M.~A. Davenport, Y.~Plan, E.~Van Den~Berg, and M.~Wootters.
\newblock 1-bit matrix completion.
\newblock \emph{Information and Inference: A Journal of the IMA}, 3\penalty0 (3):\penalty0 189--223, 2014.

\bibitem[Deng et~al.(2020)Deng, Li, Han, Shi, and Xie]{deng2020model}
L.~Deng, G.~Li, S.~Han, L.~Shi, and Y.~Xie.
\newblock Model compression and hardware acceleration for neural networks: A comprehensive survey.
\newblock \emph{Proceedings of the IEEE}, 108\penalty0 (4):\penalty0 485--532, 2020.

\bibitem[Denton et~al.(2014)Denton, Zaremba, Bruna, LeCun, and Fergus]{denton2014exploiting}
E.~L. Denton, W.~Zaremba, J.~Bruna, Y.~LeCun, and R.~Fergus.
\newblock Exploiting linear structure within convolutional networks for efficient evaluation.
\newblock \emph{Advances in neural information processing systems}, 27, 2014.

\bibitem[Du et~al.(2018)Du, Hu, and Lee]{du2018algorithmic}
S.~S. Du, W.~Hu, and J.~D. Lee.
\newblock Algorithmic regularization in learning deep homogeneous models: Layers are automatically balanced.
\newblock \emph{Advances in neural information processing systems}, 31, 2018.

\bibitem[Eckart and Young(1936)]{eckart1936approximation}
C.~Eckart and G.~Young.
\newblock The approximation of one matrix by another of lower rank.
\newblock \emph{Psychometrika}, 1\penalty0 (3):\penalty0 211--218, 1936.

\bibitem[Fazel(2002)]{fazel2002matrix}
M.~Fazel.
\newblock \emph{Matrix rank minimization with applications}.
\newblock PhD thesis, PhD thesis, Stanford University, 2002.

\bibitem[Fazel et~al.(2008)Fazel, Cand{\`e}s, Recht, and Parrilo]{fazel2008compressed}
M.~Fazel, E.~Cand{\`e}s, B.~Recht, and P.~Parrilo.
\newblock Compressed sensing and robust recovery of low rank matrices.
\newblock In \emph{2008 42nd Asilomar Conference on Signals, Systems and Computers}, pages 1043--1047. IEEE, 2008.

\bibitem[Gillis and Glineur(2011)]{gillis2011low}
N.~Gillis and F.~Glineur.
\newblock Low-rank matrix approximation with weights or missing data is np-hard.
\newblock \emph{SIAM Journal on Matrix Analysis and Applications}, 32\penalty0 (4):\penalty0 1149--1165, 2011.

\bibitem[Goldstein et~al.(2018)Goldstein, Minsker, and Wei]{goldstein2018structured}
L.~Goldstein, S.~Minsker, and X.~Wei.
\newblock Structured signal recovery from non-linear and heavy-tailed measurements.
\newblock \emph{IEEE Transactions on Information Theory}, 64\penalty0 (8):\penalty0 5513--5530, 2018.

\bibitem[Han et~al.(2015)Han, Pool, Tran, and Dally]{han2015learning}
S.~Han, J.~Pool, J.~Tran, and W.~Dally.
\newblock Learning both weights and connections for efficient neural network.
\newblock \emph{Advances in neural information processing systems}, 28, 2015.

\bibitem[Hassibi et~al.(1993)Hassibi, Stork, and Wolff]{hassibi1993optimal}
B.~Hassibi, D.~G. Stork, and G.~J. Wolff.
\newblock Optimal brain surgeon and general network pruning.
\newblock In \emph{IEEE international conference on neural networks}, pages 293--299. IEEE, 1993.

\bibitem[Huh et~al.(2021)Huh, Mobahi, Zhang, Cheung, Agrawal, and Isola]{huh2021low}
M.~Huh, H.~Mobahi, R.~Zhang, B.~Cheung, P.~Agrawal, and P.~Isola.
\newblock The low-rank simplicity bias in deep networks.
\newblock \emph{arXiv preprint arXiv:2103.10427}, 2021.

\bibitem[Jacob et~al.(2018)Jacob, Kligys, Chen, Zhu, Tang, Howard, Adam, and Kalenichenko]{jacob2018quantization}
B.~Jacob, S.~Kligys, B.~Chen, M.~Zhu, M.~Tang, A.~Howard, H.~Adam, and D.~Kalenichenko.
\newblock Quantization and training of neural networks for efficient integer-arithmetic-only inference.
\newblock In \emph{Proceedings of the IEEE conference on computer vision and pattern recognition}, pages 2704--2713, 2018.

\bibitem[Jacques et~al.(2013)Jacques, Laska, Boufounos, and Baraniuk]{jacques2013robust}
L.~Jacques, J.~N. Laska, P.~T. Boufounos, and R.~G. Baraniuk.
\newblock Robust 1-bit compressive sensing via binary stable embeddings of sparse vectors.
\newblock \emph{IEEE transactions on information theory}, 59\penalty0 (4):\penalty0 2082--2102, 2013.

\bibitem[Jaderberg et~al.(2014)Jaderberg, Vedaldi, and Zisserman]{jaderberg2014speeding}
M.~Jaderberg, A.~Vedaldi, and A.~Zisserman.
\newblock Speeding up convolutional neural networks with low rank expansions.
\newblock \emph{arXiv preprint arXiv:1405.3866}, 2014.

\bibitem[Jain et~al.(2010)Jain, Meka, and Dhillon]{jain2010guaranteed}
P.~Jain, R.~Meka, and I.~Dhillon.
\newblock Guaranteed rank minimization via singular value projection.
\newblock \emph{Advances in Neural Information Processing Systems}, 23, 2010.

\bibitem[Jain et~al.(2013)Jain, Netrapalli, and Sanghavi]{jain2013low}
P.~Jain, P.~Netrapalli, and S.~Sanghavi.
\newblock Low-rank matrix completion using alternating minimization.
\newblock In \emph{Proceedings of the forty-fifth annual ACM symposium on Theory of computing}, pages 665--674, 2013.

\bibitem[Jiang et~al.(2023)Jiang, Chen, and Ding]{jiang2023algorithmic}
L.~Jiang, Y.~Chen, and L.~Ding.
\newblock Algorithmic regularization in model-free overparametrized asymmetric matrix factorization.
\newblock \emph{SIAM Journal on Mathematics of Data Science}, 5\penalty0 (3):\penalty0 723--744, 2023.

\bibitem[Kitaev et~al.(2020)Kitaev, Kaiser, and Levskaya]{kitaev2020reformer}
N.~Kitaev, {\L}.~Kaiser, and A.~Levskaya.
\newblock Reformer: The efficient transformer.
\newblock \emph{arXiv preprint arXiv:2001.04451}, 2020.

\bibitem[Lebedev et~al.(2014)Lebedev, Ganin, Rakhuba, Oseledets, and Lempitsky]{lebedev2014speeding}
V.~Lebedev, Y.~Ganin, M.~Rakhuba, I.~Oseledets, and V.~Lempitsky.
\newblock Speeding-up convolutional neural networks using fine-tuned cp-decomposition.
\newblock \emph{arXiv preprint arXiv:1412.6553}, 2014.

\bibitem[Li and Shi(2018)]{li2018constrained}
C.~Li and C.~Shi.
\newblock Constrained optimization based low-rank approximation of deep neural networks.
\newblock In \emph{Proceedings of the European Conference on Computer Vision (ECCV)}, pages 732--747, 2018.

\bibitem[Li et~al.(2018)Li, Ma, and Zhang]{li2018algorithmic}
Y.~Li, T.~Ma, and H.~Zhang.
\newblock Algorithmic regularization in over-parameterized matrix sensing and neural networks with quadratic activations.
\newblock In \emph{Conference On Learning Theory}, pages 2--47. PMLR, 2018.

\bibitem[Liang et~al.(2023)Liang, Jiang, Li, Tang, Yin, and Zhao]{liang2023homodistil}
C.~Liang, H.~Jiang, Z.~Li, X.~Tang, B.~Yin, and T.~Zhao.
\newblock Homodistil: Homotopic task-agnostic distillation of pre-trained transformers.
\newblock \emph{arXiv preprint arXiv:2302.09632}, 2023.

\bibitem[Liu and Parhi(2023)]{liu2023tensor}
X.~Liu and K.~K. Parhi.
\newblock Tensor decomposition for model reduction in neural networks: A review.
\newblock \emph{arXiv preprint arXiv:2304.13539}, 2023.

\bibitem[Ludoux and Talagrand(1991)]{ludoux1991probability}
M.~Ludoux and M.~Talagrand.
\newblock Probability in banach spaces: Isoperimetry and processes, 1991.

\bibitem[Luo(2022)]{luo2022understanding}
C.~Luo.
\newblock Understanding diffusion models: A unified perspective.
\newblock \emph{arXiv preprint arXiv:2208.11970}, 2022.

\bibitem[Mohan and Fazel(2010)]{mohan2010new}
K.~Mohan and M.~Fazel.
\newblock New restricted isometry results for noisy low-rank recovery.
\newblock In \emph{2010 IEEE International Symposium on Information Theory}, pages 1573--1577. IEEE, 2010.

\bibitem[Negahban and Wainwright(2012)]{negahban2012restricted}
S.~Negahban and M.~J. Wainwright.
\newblock Restricted strong convexity and weighted matrix completion: Optimal bounds with noise.
\newblock \emph{The Journal of Machine Learning Research}, 13:\penalty0 1665--1697, 2012.

\bibitem[Neill(2020)]{neill2020overview}
J.~O. Neill.
\newblock An overview of neural network compression.
\newblock \emph{arXiv preprint arXiv:2006.03669}, 2020.

\bibitem[Nguyen et~al.(2019)Nguyen, Kim, and Shim]{nguyen2019low}
L.~T. Nguyen, J.~Kim, and B.~Shim.
\newblock Low-rank matrix completion: A contemporary survey.
\newblock \emph{IEEE Access}, 7:\penalty0 94215--94237, 2019.

\bibitem[Nie et~al.(2023)Nie, Wang, and Zheng]{nie2023low}
J.~Nie, L.~Wang, and Z.~Zheng.
\newblock Low rank tensor decompositions and approximations.
\newblock \emph{Journal of the Operations Research Society of China}, pages 1--27, 2023.

\bibitem[Pan et~al.(2020)Pan, Ling, He, Qi, and Xu]{pan2020low}
C.~Pan, C.~Ling, H.~He, L.~Qi, and Y.~Xu.
\newblock Low-rank and sparse enhanced tucker decomposition for tensor completion.
\newblock \emph{arXiv preprint arXiv:2010.00359}, 2020.

\bibitem[Papadimitriou and Jain(2021)]{papadimitriou2021data}
D.~Papadimitriou and S.~Jain.
\newblock Data-driven low-rank neural network compression.
\newblock In \emph{2021 IEEE International Conference on Image Processing (ICIP)}, pages 3547--3551. IEEE, 2021.

\bibitem[Papyan et~al.(2020)Papyan, Han, and Donoho]{papyan2020prevalence}
V.~Papyan, X.~Han, and D.~L. Donoho.
\newblock Prevalence of neural collapse during the terminal phase of deep learning training.
\newblock \emph{Proceedings of the National Academy of Sciences}, 117\penalty0 (40):\penalty0 24652--24663, 2020.

\bibitem[Plan and Vershynin(2016)]{plan2016generalized}
Y.~Plan and R.~Vershynin.
\newblock The generalized lasso with non-linear observations.
\newblock \emph{IEEE Transactions on information theory}, 62\penalty0 (3):\penalty0 1528--1537, 2016.

\bibitem[Plan et~al.(2017)Plan, Vershynin, and Yudovina]{plan2017high}
Y.~Plan, R.~Vershynin, and E.~Yudovina.
\newblock High-dimensional estimation with geometric constraints.
\newblock \emph{Information and Inference: A Journal of the IMA}, 6\penalty0 (1):\penalty0 1--40, 2017.

\bibitem[Price and Tanner(2023)]{price2023improved}
I.~Price and J.~Tanner.
\newblock Improved projection learning for lower dimensional feature maps.
\newblock In \emph{ICASSP 2023-2023 IEEE International Conference on Acoustics, Speech and Signal Processing (ICASSP)}, pages 1--5. IEEE, 2023.

\bibitem[Raskutti et~al.(2019)Raskutti, Yuan, and Chen]{raskutti2019convex}
G.~Raskutti, M.~Yuan, and H.~Chen.
\newblock Convex regularization for high-dimensional multiresponse tensor regression.
\newblock 2019.

\bibitem[Recht and R{\'e}(2013)]{recht2013parallel}
B.~Recht and C.~R{\'e}.
\newblock Parallel stochastic gradient algorithms for large-scale matrix completion.
\newblock \emph{Mathematical Programming Computation}, 5\penalty0 (2):\penalty0 201--226, 2013.

\bibitem[Seginer(2000)]{seginer2000expected}
Y.~Seginer.
\newblock The expected norm of random matrices.
\newblock \emph{Combinatorics, Probability and Computing}, 9\penalty0 (2):\penalty0 149--166, 2000.

\bibitem[Seleznova et~al.(2024)Seleznova, Weitzner, Giryes, Kutyniok, and Chou]{seleznova2024neural}
M.~Seleznova, D.~Weitzner, R.~Giryes, G.~Kutyniok, and H.-H. Chou.
\newblock Neural (tangent kernel) collapse.
\newblock \emph{Advances in Neural Information Processing Systems}, 36, 2024.

\bibitem[Stewart(1998)]{stewart1998perturbation}
G.~W. Stewart.
\newblock Perturbation theory for the singular value decomposition.
\newblock Technical report, 1998.

\bibitem[Tanner and Wei(2016)]{tanner2016low}
J.~Tanner and K.~Wei.
\newblock Low rank matrix completion by alternating steepest descent methods.
\newblock \emph{Applied and Computational Harmonic Analysis}, 40\penalty0 (2):\penalty0 417--429, 2016.

\bibitem[Thrampoulidis et~al.(2015)Thrampoulidis, Abbasi, and Hassibi]{thrampoulidis2015lasso}
C.~Thrampoulidis, E.~Abbasi, and B.~Hassibi.
\newblock Lasso with non-linear measurements is equivalent to one with linear measurements.
\newblock \emph{Advances in Neural Information Processing Systems}, 28, 2015.

\bibitem[Timor et~al.(2023)Timor, Vardi, and Shamir]{timor2023implicit}
N.~Timor, G.~Vardi, and O.~Shamir.
\newblock Implicit regularization towards rank minimization in relu networks.
\newblock In \emph{International Conference on Algorithmic Learning Theory}, pages 1429--1459. PMLR, 2023.

\bibitem[Vershynin(2020)]{vershynin2020high}
R.~Vershynin.
\newblock High-dimensional probability.
\newblock \emph{University of California, Irvine}, 10:\penalty0 11, 2020.

\bibitem[Wang and Cheng(2016)]{wang2016accelerating}
P.~Wang and J.~Cheng.
\newblock Accelerating convolutional neural networks for mobile applications.
\newblock In \emph{Proceedings of the 24th ACM international conference on Multimedia}, pages 541--545, 2016.

\bibitem[Xia et~al.(2022)Xia, Zhang, and Zhou]{xia2022inference}
D.~Xia, A.~R. Zhang, and Y.~Zhou.
\newblock Inference for low-rank tensors—no need to debias.
\newblock \emph{The Annals of Statistics}, 50\penalty0 (2):\penalty0 1220--1245, 2022.

\bibitem[Xu and McAuley(2023)]{xu2023survey}
C.~Xu and J.~McAuley.
\newblock A survey on model compression and acceleration for pretrained language models.
\newblock In \emph{Proceedings of the AAAI Conference on Artificial Intelligence}, volume~37, pages 10566--10575, 2023.

\bibitem[Yu and Wu(2023)]{yu2023compressing}
H.~Yu and J.~Wu.
\newblock Compressing transformers: features are low-rank, but weights are not!
\newblock In \emph{Proceedings of the AAAI Conference on Artificial Intelligence}, volume~37, pages 11007--11015, 2023.

\bibitem[Yuan and Zhang(2016)]{yuan2016tensor}
M.~Yuan and C.-H. Zhang.
\newblock On tensor completion via nuclear norm minimization.
\newblock \emph{Foundations of Computational Mathematics}, 16\penalty0 (4):\penalty0 1031--1068, 2016.

\bibitem[Zhang and Saab(2023)]{zhang2023spfq}
J.~Zhang and R.~Saab.
\newblock Spfq: A stochastic algorithm and its error analysis for neural network quantization.
\newblock \emph{arXiv preprint arXiv:2309.10975}, 2023.

\bibitem[Zhang et~al.(2023)Zhang, Zhou, and Saab]{zhang2023post}
J.~Zhang, Y.~Zhou, and R.~Saab.
\newblock Post-training quantization for neural networks with provable guarantees.
\newblock \emph{SIAM Journal on Mathematics of Data Science}, 5\penalty0 (2):\penalty0 373--399, 2023.

\bibitem[Zhang et~al.(2015)Zhang, Zou, He, and Sun]{zhang2015accelerating}
X.~Zhang, J.~Zou, K.~He, and J.~Sun.
\newblock Accelerating very deep convolutional networks for classification and detection.
\newblock \emph{IEEE transactions on pattern analysis and machine intelligence}, 38\penalty0 (10):\penalty0 1943--1955, 2015.

\bibitem[Zhu et~al.(2023)Zhu, Li, Liu, Ma, and Wang]{zhu2023survey}
X.~Zhu, J.~Li, Y.~Liu, C.~Ma, and W.~Wang.
\newblock A survey on model compression for large language models.
\newblock \emph{arXiv preprint arXiv:2308.07633}, 2023.

\bibitem[Zuk and Wagner(2015)]{zuk2015low}
O.~Zuk and A.~Wagner.
\newblock Low-rank matrix recovery from row-and-column affine measurements.
\newblock In \emph{International Conference on Machine Learning}, pages 2012--2020. PMLR, 2015.

\end{thebibliography}

\appendix

\if{\section{Supporting Evidence for Motivation}\label{sec:evidence for motivation}
\begin{figure}[htbp]
  \centering
  \includegraphics[width=1\textwidth]{Low Rank Implicit Bias Comparison.png}
  \caption{Low Rank Implicit Bias}
  \label{Low Rank Bias Graph}
\end{figure}


Model training generally follows two approaches: direct training on the target dataset or pretraining on a much larger dataset followed by fine-tuning for the specific task. For instance, ViT models \cite{dosovitskiy2020image} are pretrained on the large, private JFT-300M dataset \cite{sun2017revisiting} before being transferred to smaller tasks like ImageNet classification \cite{deng2009imagenet}, while DeiT(3) \cite{touvron2021training, touvron2022deit} models are trained solely on ImageNet, without any external data. Based on our observation that the low-rank structure arises from the interaction between input data and weights, we expect that training directly on the target dataset will lead to stronger low-rankness, or a smaller effective rank, compared to pretraining on a different, larger dataset. This expectation is supported by the following numerical results.

We compare ViT models with DeiT(3) models across three different scales: small, base, and large. At each layer, we reshape the input data to $BP \times d$, where $B$ represents the batch size, $P$ denotes the number of patches, and $d$ is the embedding dimension. Each transformer block comprises two components: an attention layer and a multi-layer perceptron (MLP). Each component contains two fully connected layers, referred to as the qkv layer and the proj layer in the attention layer, and FC1 and FC2 in the MLP, following the variable naming conventions used by Hugging Face. The plots \cref{Low Rank Bias Graph} illustrate the rank retained in all the qkv and FC1 layers when 90\% of the total energy is preserved. Here, total energy is defined as the nuclear norm—the sum of all singular values—of the pre-activation $XW$, and the selected rank is the minimum number of top singular values needed to retain at least 90\% of this total energy. The even x-indices correspond to the qkv layers in the models, while the odd x-indices correspond to the FC1 layers. The y-axis represents the selected rank.

From \cref{Low Rank Bias Graph}, we observe that DeiT(3) models generally have smaller selected ranks than ViT models across most layers, with closely matched selected ranks in the remaining layers. Notably, DeiT\_base exhibits a smaller selected rank than ViT\_base for all layers except the last FC1 layer. This finding reinforces the interpretation of approximate low-rankness in neural networks as a collaborative effect between the input data and the trained weights.
}\fi

\section{Lemmas for Theorem \ref{thm:nonlinear recovery}}\label{sec:lemmas for nonlinear theorem}
~\\
We start with two standard lemmas from the literature, see, e.g., \cite{ludoux1991probability}.
\begin{lemma}[Contraction, \cite{ludoux1991probability}, Theorem 4.12]\label{contraction}
Let $F: \mathbb{R}_+\rightarrow\mathbb{R}_+$ be convex and increasing. Let $\varphi_i: \mathbb{R}\rightarrow\mathbb{R}, i\leq N$ be contractions (1-Lipschitz functions) such that $\varphi_i(0)=0$. If $h(t)$ is a function on some set T, we define $\|h\|_T=sup_{t\in T}|h(t)|$. Then for any bounded set $T \subset \mathbb{R}^N$ and $(\epsilon_i)_{i=1}^N$ an i.i.d Rademacher sequence, we have
$$
\mathbb{E}F(\frac{1}{2}\|\sum_{i=1}^N \epsilon_i \varphi_i(t_i)\|_T)\leq \mathbb{E}F(\|\sum_{i=1}^N \epsilon_i t_i\|_T).
$$
\end{lemma}
\begin{lemma}[Symmetrization, \cite{ludoux1991probability}, Lemma 6.3]\label{symmetrization}
Let $(B,\|\cdot\|)$ be a separable Banach space.
Let $F: \mathbb{R}_+\rightarrow\mathbb{R}_+$ be convex. Then, for any finite sequence $(X_i)$ of independent mean zero Borel random variables taking value in $B$ such that $\mathbb{E}F(\|X_i\|)<\infty$ for every $i$, and $(\epsilon_i)$ an i.i.d. Rademacher sequence which is independent of $(X_i)$, we have
$$
\mathbb{E}F(\frac{1}{2}\|\sum_i \epsilon_i X_i\|)\leq\mathbb{E}F(\|\sum_i X_i\|)\leq\mathbb{E}F(2\|\sum_i \epsilon_i X_i\|).
$$
If $(X_i)$ is not necessarily mean zero, we have
$$
\mathbb{E}F(\sup_{f\in D}|\sum_i f(X_i)-\mathbb{E}f(X_i)|)\leq \mathbb{E}F(2\|\sum_i \epsilon_i X_i\|)
$$
and
$$
\mathbb{E}F(\sup_{f\in D}|\sum_i\epsilon_i (f(X_i)-\mathbb{E}f(X_i))|)\leq \mathbb{E}F(2\|\sum_i X_i\|),
$$
where $D$ is the unit ball of the dual space of $B$.
\end{lemma}
~\\
The next lemma, whose proof we provide for completeness, is also a standard estimate for the maximum entry of a random matrix.

\begin{lemma}\label{gaussian infty}(Max Entry Estimate of Gaussian Matrix)
Let $G\in \mathbb{R}^{d_1 \times d_2}$ be a random Gaussian matrix with i.i.d $\mathcal{N}(0, \sigma^2)$, then $\mathbb{P}(max_{ij}G_{ij}\geq 2\sqrt{\log(d_1d_2)}\sigma)\leq \frac{1}{2\sqrt{2\pi}}\frac{1}{d_1d_2\sqrt{\log(d_1d_2)}}$.
\begin{proof}
    Let $g\sim \mathcal{N}(0, \sigma^2)$. We have the basic tail estimate for normal random variables\cite{vershynin2020high},  namely that for any $t>0$,
    $
        \mathbb{P}(g\geq t\sigma)\leq \frac{1}{t}\frac{1}{\sqrt{2\pi}}e^{-\frac{t^2}{2}}.
    $
    Consequently, by a union bound
    \begin{align*}
        \mathbb{P}(\max_{ij}G_{ij}\geq t\sigma)
        \leq d_1d_2 \frac{1}{t}\frac{1}{\sqrt{2\pi}}e^{-\frac{t^2}{2}}.
    \end{align*}
    Thus, picking $t=2\sqrt{\log(d_1d_2)}$, we get the desired result.
\end{proof}
\end{lemma}

\begin{lemma}[\cite{seginer2000expected}, Theorem 1.1.]\label{Rade norm}
    Let $\mathbf{E}\in \mathbb{R}^{d_1 \times d_2}$ be a matrix whose entries are i.i.d. Rademacher random variables $\epsilon_{ij}$, and let $h>0$. Then there exists an absolute constant $K$ independent of the dimensions and $h$, such that$$
    \mathbb{E}[\|\mathbf{E}\|^h]\leq K (\sqrt{2(d_1+d_2)})^h
    .$$
    Here the norm on $\mathbf{E}$ is the operator norm.
\end{lemma}

\begin{definition}[Hellinger distance]
    For two scalars $p, q \in [0,1]$, the Hellinger distance is given by$$
    d_H^2(p,q):=(\sqrt{p}-\sqrt{q})^2+(\sqrt{1-p}-\sqrt{1-q})^2
    .$$
    This defines a distance between two binary probability distributions. This definition can be easily extended to matrices via the average Hellinger distance over all entries. For matrices $\mathbf{P}, \mathbf{Q} \in [0,1]^{d_1\times d_2} $, the Hellinger distance is given by$$
    d_H^2(\mathbf{P}, \mathbf{Q}):=\frac{1}{d_1d_2}\sum_{i, j}d_H^2(P_{i,j},Q_{i,j})
    .$$
\end{definition}
\begin{definition}[KL divergence]
    For two scalars $p, q \in [0,1]$, the Kullback–Leibler (KL) divergence is defined by$$
    D_{KL}(p \| q):=p\log(\frac{p}{q})+(1-p)\log(\frac{1-p}{1-q})
    .$$
    For matrices $\mathbf{P}, \mathbf{Q} \in [0,1]^{d_1\times d_2} $, we define the KL divergence to be$$
    D_{KL}(\mathbf{P} \| \mathbf{Q}):=\frac{1}{d_1d_2}\sum_{i, j}D_{KL}(P_{i,j},Q_{i,j})
.    $$
\end{definition}

{We end this section with a well-known result that Hellinger distance can be bounded above by KL divergence which is also used in \cite{davenport20141}.}

\begin{lemma}\label{HelltoKL}
    For two scalars $p, q \in [0,1]$, we have $d_H^2(p,q)\leq D_{KL}(p \| q)$. Therefore, $d_H^2(\mathbf{P}, \mathbf{Q}) \leq D_{KL}(\mathbf{P} \| \mathbf{Q})$ for matrices $\mathbf{P}, \mathbf{Q} \in [0,1]^{d_1\times d_2} $.
    \begin{proof}
        The proof is based on a simple observation that $-\log(x)\geq 1-x$ when $x\in (0,1]$. Indeed,
        \begin{align*}
            D_{KL}(p \| q)&=p\log\left(\frac{p}{q}\right)+(1-p)\log\left(\frac{1-p}{1-q}\right)
            =2\left[p\left(-\log\sqrt{\frac{q}{p}}\right)+(1-p)\left(-\log\sqrt{\frac{1-q}{1-p}}\right)\right]\\&
            \geq 2\left[p\left(1-\sqrt{\frac{q}{p}}\right)+(1-p)\left(1-\sqrt{\frac{1-q}{1-p}}\right)\right]
            =(\sqrt{p}-\sqrt{q})^2+(\sqrt{1-p}-\sqrt{1-q})^2\\
            &=d_H^2(p,q)
        \end{align*}
    \end{proof}
\end{lemma}

\section{Proof of Theorem \ref{thm:nonlinear recovery}}\label{sec:prove nonlinear theorem}

Let's start with a few technical lemmas. Throughout this section, $\phi$ will denote the probability density function (PDF) of standard normal distribution, i.e., $\phi(x)=\frac{1}{\sqrt{2\pi}} e^{-\frac{x^2}{2}}$, and $\Phi$ will denote its cumulative distribution function (CDF), i.e. $\Phi(x) = \int_{-\infty}^{x} \frac{1}{\sqrt{2\pi}} e^{-\frac{t^2}{2}} dt$. Recall that we use bold letter $\mathbf{\Phi}$ to denote an MLP, but it will also be easy to distinguish them from context. 

\begin{lemma}\label{sod bound}
Let $f(x)$ be the CDF of the normal distribution $\mathcal{N}(0, \sigma^2)$. Then 
\[
0 \geq \log(f(x))'' \geq -\frac{1}{\sigma^2} \quad \text{for all } x \in \mathbb{R}.
\]
\begin{proof}
The inequality $0 \geq \log(f(x))''$ follows from the well-known fact that the CDF $f(x)$ of a normal distribution is log-concave (see, e.g. \cite{borzadaran2011log}). We focus here on proving the other inequality $\log(f(x))'' \geq -\frac{1}{\sigma^2}$~\\

Direct calculation yields 
$
\log(f(x))' = \frac{f'(x)}{f(x)}, \quad \log(f(x))'' = \frac{f''(x)f(x) - f'(x)^2}{f(x)^2}
$, so substituting $f(x) = \Phi\left(\frac{x}{\sigma}\right)$ and $f'(x) = \frac{1}{\sigma}\phi\left(\frac{x}{\sigma}\right)$, we have 
\[
\log(f(x))'' = \frac{\frac{1}{\sigma^2}\phi'\left(\frac{x}{\sigma}\right)\Phi\left(\frac{x}{\sigma}\right) - \frac{1}{\sigma^2}\phi\left(\frac{x}{\sigma}\right)^2}{\Phi\left(\frac{x}{\sigma}\right)^2}.
\]
Thus, $\log(f(x))'' \geq -\frac{1}{\sigma^2}$ for all $x \in \mathbb{R}$ is equivalent to 
\[
\frac{\phi'\left(\frac{x}{\sigma}\right)\Phi\left(\frac{x}{\sigma}\right) - \phi\left(\frac{x}{\sigma}\right)^2}{\Phi\left(\frac{x}{\sigma}\right)^2} \geq -1.
\]
It suffices to prove the result for the standard normal distribution, i.e., to show 
\[
\frac{\phi'(x)\Phi(x) - \phi(x)^2}{\Phi(x)^2} \geq -1 \quad \text{for all } x \in \mathbb{R}.
\]
Using $\phi'(x) = -x\phi(x)$, we rewrite this as 
\[
g(x) = -x\phi(x)\Phi(x) - \phi(x)^2 + \Phi(x)^2 \geq 0.
\]
It is straightforward to verify that as $x \to -\infty$, $g(x) \to 0$. To conclude $g(x) \geq 0$, we will show $g(x)$ is monotonically increasing, i.e., $g'(x) \geq 0$. To that end, we compute 
\[
g'(x) = (1 + x^2)\phi(x)\Phi(x) + x\phi(x)^2.
\]
Since $\phi(x) > 0$ and $1 > \Phi(x) > 0$, it is clear that $g'(x) > 0$ when $x \geq 0$.

For $x < 0$, let $x = -y$ with $y > 0$. Using $\phi(-y) = \phi(y)$ and $\Phi(-y) = 1 - \Phi(y)$, we rewrite 
\[
g'(x) \geq 0 \iff (1 + y^2)\phi(-y)\Phi(-y) - y\phi(-y)^2 \geq 0.
\]
Substituting $\phi(-y) = \phi(y)$ and $\Phi(-y) = 1 - \Phi(y)$, this becomes 
\[
(1 + y^2)(1 - \Phi(y)) \geq y\phi(y),
\]
which simplifies to 
\[
h(y) = \frac{1}{\sqrt{2\pi}} \int_{y}^{\infty} e^{-\frac{t^2}{2}} dt - \frac{y}{1 + y^2} \frac{1}{\sqrt{2\pi}} e^{-\frac{y^2}{2}} \geq 0.
\]
Clearly, $h(0) = \frac{1}{2}$ and $\lim_{y \to \infty} h(y) = 0$. To show $h(y) \geq 0$ for all $y > 0$, we compute 
\[
h'(y) = \frac{1}{\sqrt{2\pi}} e^{-\frac{y^2}{2}} \left(-1 + \frac{y^2 - 1}{(1 + y^2)^2} + \frac{y^2}{1 + y^2}\right)
= -\frac{2}{\sqrt{2\pi}} e^{-\frac{y^2}{2}} \frac{1}{(1 + y^2)^2}.
\]

Thus, $h(y)$ is monotonically decreasing for $y > 0$, and $h(y) \geq 0$ for all $y \geq 0$. This completes the proof.
\end{proof}
\end{lemma}

\if{
\begin{lemma}\label{sod bound}(Second Order Derivative Bound)~\\
Let $f(x)$ be the CDF of normal distribution $\mathcal{N}(0, \sigma^2)$, then  $0 \geq \log(f(x))'' \geq -\frac{1}{\sigma^2}$ for all $x \in \mathbb{R}$.
\begin{proof}
    $0 \geq \log(f(x))''$ follows from the well-known fact that the CDF $f(x)$ of a normal distribution is log-concave, see e.g. \cite{borzadaran2011log}. We only prove the other inequality $\log(f(x))'' \geq -\frac{1}{\sigma^2}$ here.~\\
    Direct calculation yields $\log(f(x))'=\frac{f'(x)}{f(x)}$, and $\log(f(x))''=\frac{f''(x)f(x)-f'(x)^2}{f(x)^2}$.~\\
    Now we plug in $f(x)=\Phi(\frac{x}{\sigma})$, and $f'(x)=\frac{1}{\sigma}\phi(\frac{x}{\sigma})$ to get:
    $$
    \log(f(x))''=\frac{\frac{1}{\sigma^2}\phi'(\frac{x}{\sigma})\Phi(\frac{x}{\sigma})-\frac{1}{\sigma^2}\phi(\frac{x}{\sigma})^2}{\Phi(\frac{x}{\sigma})^2}.
    $$
    Thus $\log(f(x))'' \geq -\frac{1}{\sigma^2}$ for all $x \in \mathbb{R}$ is equivalent to: 
    $$
    \frac{\phi'(\frac{x}{\sigma})\Phi(\frac{x}{\sigma})-\phi(\frac{x}{\sigma})^2}{\Phi(\frac{x}{\sigma})^2} \geq -1
    $$ for all $x \in \mathbb{R}$. So it is sufficient to prove the result only for standard normal distribution, i.e. to show $\frac{\phi'(x)\Phi(x)-\phi(x)^2}{\Phi(x)^2} \geq -1$  for all $x \in \mathbb{R}$.~\\
    Use the fact $\phi'(x)=-x\phi(x)$ and rearrange the terms to get $\frac{\phi'(x)\Phi(x)-\phi(x)^2}{\Phi(x)^2} \geq -1$ is equivalent to $g(x):=-x\phi(x)\Phi(x)-\phi(x)^2+\Phi(x)^2\geq0$. It is an easy check that when $x\rightarrow -\infty$, $g(x)\rightarrow0$ so $g(x)\geq0$ will be implied by proving $g(x)$ is monotonely increasing, i.e. $g'(x)\geq0$.
    \begin{align*}
        g'(x)&=-\phi(x)\Phi(x)-x\phi'(x)\Phi(x)-x\phi(x)^2-2\phi(x)\phi'(x)+2\Phi(x)\phi(x) \\
        &=-\phi(x)\Phi(x)-x(-x\phi(x))\Phi(x)-x\phi(x)^2-2\phi(x)(-x\phi(x))+2\Phi(x)\phi(x) \\
        &=(1+x^2)\phi(x)\Phi(x)+x\phi(x)^2
    \end{align*}
    Since $\phi(x)>0$ and $1>\Phi(x)>0$, when $x\geq0$ it is obvious that $g'(x)>0$. ~\\
    When $x<0$, let $x=-y$, $y>0$. Note that $\phi(-y)=\phi(y)$ and $\Phi(-y)=1-\Phi(y)$.
    \begin{align*}
        g'(x)\geq0 &\Longleftrightarrow (1+y^2)\phi(-y)\Phi(-y)-y\phi(-y)^2\geq0 \\
        &\Longleftrightarrow (1+y^2)(1-\Phi(y))\geq y\phi(y)\\
        &\Longleftrightarrow \int_{y}^{\infty} \frac{1}{\sqrt{2\pi}} e^{-\frac{t^2}{2}} dt\geq \frac{y}{1+y^2}\frac{1}{\sqrt{2\pi}} e^{-\frac{y^2}{2}} \\
        &\Longleftrightarrow h(y):= \frac{1}{\sqrt{2\pi}}\int_{y}^{\infty} e^{-\frac{t^2}{2}} dt - \frac{y}{1+y^2} \frac{1}{\sqrt{2\pi}}e^{-\frac{y^2}{2}}\geq 0
    \end{align*}
    Clearly $h(0)=\frac{1}{2}$ and $\lim_{x \to \infty} h(y) = 0$. In order to show that $h(y)\geq 0$ for all $y>0$ it suffices to show that $h(y)$ is monotonely decreasing on the positive real line.
    \begin{align*}
        h'(y) &= \frac{1}{\sqrt{2\pi}} (-e^{-\frac{y^2}{2}}-\frac{(1+y^2)-y(2y)}{(1+y^2)^2} e^{-\frac{y^2}{2}} - \frac{y}{1+y^2}(-y)e^{-\frac{y^2}{2}}) \\
        &= \frac{1}{\sqrt{2\pi}}e^{-\frac{y^2}{2}}(-1+\frac{y^2-1}{(1+y^2)^2}+\frac{y^2}{1+y^2}) \\
        &= -\frac{2}{\sqrt{2\pi}}e^{-\frac{y^2}{2}}\frac{1}{(1+y^2)^2}\\
        &<0
    \end{align*}
    From the above calculation we can conclude $h(y)$ is indeed decreasing on the whole real line and $h(y)\geq 0$ for all $y\in \mathbb{R}$ which finishes the proof.
\end{proof}
\end{lemma}
}\fi
\begin{corollary}\label{logtalyor}
Let f be as in  \cref{sod bound}, then $\log(f(b))-\log(f(a))\geq \frac{f'(a)}{f(a)}(b-a)-\frac{1}{2\sigma^2}(b-a)^2$.
\begin{proof}
    Apply a Taylor expansion to the function $\log(f(x))$ at $x=a$ and use the lower bound in lemma \ref{sod bound}.
\end{proof}
\end{corollary}
\begin{lemma}[\cite{davenport20141}]\label{constants}
Let f be as in lemma \ref{sod bound}, then the two constants: $L_{\alpha, \sigma}:=\underset{|x|\leq \alpha}{\sup} \frac{|f'(x)|}{f(x)(1-f(x))}$ and $\beta_{\alpha, \sigma}:=\underset{|x|\leq \alpha}{\sup} \frac{f(x)(1-f(x))}{f'(x)^2}$ satisfy $L_{\alpha, \sigma}\leq 8 \frac{\alpha + \sigma}{\sigma^2}$ and $\beta_{\alpha, \sigma}\leq \pi \sigma^2 e^{\alpha^2/2\sigma^2}$.
\end{lemma}

\begin{lemma}[\cite{davenport20141}, Lemma A.2]\label{HelltoFrob}
    Let $f$ be a differentiable function and let $\mathbf{M}, \mathbf{M}'$ be two matrices satisfying $\|\mathbf{M}\|_{\infty}\leq \alpha$ and $\|\mathbf{M}'\|_{\infty}\leq \alpha$. Then $$
    d_H^2(f(\mathbf{M}),f(\mathbf{M}'))\geq \frac{1}{8\beta_{\alpha}} \frac{\|\mathbf{M}-\mathbf{M}'\|_F^2}{d_1d_2}
    $$
\end{lemma}
Now we are ready to prove the theorem using techniques from \cite{davenport20141}.

\subsection*{Proof of Theorem \ref{thm:nonlinear recovery}}

\begin{proof}
Recall we defined $\Psi(\widecheck{X}):=\{Y\in \mathbb{R}^{d_1 \times d_2}: \|Y\|_*\leq\alpha \sqrt{r d_1 d_2};\ \|Y\|_{\infty}\leq \alpha ;\ Y_i\in \mathrm{span}\{\mathrm{col}(\widecheck{X})\}, i=1,......, d_2  \}$ in the proof of \cref{thm:recovery two}, and we have $\widecheck{X}\Omega = \Psi(\widecheck{X})$ and $M\in \Omega$. As we assume $d_1\geq d$ and $\widecheck{X}$ is full rank here, this is a one to one mapping between $\Omega$ and $\Psi(\widecheck{X})$. Defining $Y:=\widecheck{X}M \in\Psi(\widecheck{X}) $ as in \cref{thm:recovery two}, we have $Z=\rho(Y+G)$ and proving \cref{thm:nonlinear recovery} is reduced to proving that with high probability, the solution $\Hat{Y}$ to 
\begin{align}
   &\max_{M'}\underset{(i,j):Z_{ij}>0}{\sum}\log\left(\frac{1}{\sqrt{2\pi}\sigma}e^{-\frac{(Z_{ij}-M'_{ij})^2}{2\sigma^2}}\right)+\underset{(i,j):Z_{ij}=0}{\sum}\log(1-f(M'_{ij})) \tag{$P_*$} 
   &\text{subject to} \ M' \in \Psi(\widecheck{X}) \  
\end{align}
satisfies
\begin{equation}
\frac{1}{d_1d_2}\|Y-\Hat{Y}\|_F^2 \leq C_{\alpha, \sigma} \max\left\{2\sqrt{log(d_1d_2)}, 8\right\} \sqrt{\frac{r(d_1+d_2)}{d_1 d_2}}.
\end{equation}

For any $M' \in \Psi(\widecheck{X})$, let us denote the loss function by \begin{align*}\mathcal{L}(M'|Z)&=\underset{(i,j):Z_{ij}>0}{\sum}\log(\frac{1}{\sqrt{2\pi}\sigma}e^{-\frac{(Z_{ij}-M'_{ij})^2}{2\sigma^2}})+\underset{(i,j):Z_{ij}=0}{\sum}\log(1-f(M'_{ij}))\\&=\underset{(i,j)}{\sum}(\mathbbm{1}_{[Z_{ij}>0]} \log(\frac{1}{\sqrt{2\pi}\sigma}e^{-\frac{(Z_{ij}-M'_{ij})^2}{2\sigma^2}})+\mathbbm{1}_{[Z_{ij}=0]} \log(1-f(M'_{ij}))),\end{align*} 
and recall we are interested in the difference between the solution $\Hat{Y}$ to 
\begin{align*}
&\max_{M'} \mathcal{L}(M'|Z) \tag{$P_*$} 
   \text{\quad subject to\quad} \ M' \in \Psi(\widecheck{X})   
\end{align*}
and the ground truth $Y=\widecheck{X}M \in\Psi(\widecheck{X})$.
To that end, we may replace $\mathcal{L}(M'|Z)$ by its centered version 
\begin{align*}
\widebar{\mathcal{L}}(M'|Z)&=\mathcal{L}(M'|Z)-\mathcal{L}(\mathbf{0}|Z)\\&
=\underset{(i,j)}{\sum}(\mathbbm{1}_{[Z_{ij}>0]}\frac{-1}{2\sigma^2}(M_{ij}^{' 2}-2Z_{ij}M'_{ij})+\mathbbm{1}_{[Z_{ij}=0]} \log(\frac{1-f(M'_{ij})}{1-f(0)}))
\end{align*} without affecting the optimizer $\hat{Y}$ of ($P_*$). 

Similar to the proof technique we used for \cref{{thm:recovery two}}, we will rely on the inequalities
\begin{align*}
   0 \leq \widebar{\mathcal{L}}(\hat{Y}|Z)-\widebar{\mathcal{L}}(Y|Z)
   & \leq  \mathbb{E}[\widebar{\mathcal{L}}(\hat{Y}|Z)-\widebar{\mathcal{L}}(Y|Z)]+2\underset{M' \in \Psi(\widecheck{X})}{\sup}|\widebar{\mathcal{L}}(M'|Z)-\mathbb{E}[\widebar{\mathcal{L}}(M'|Z)]|, 
\end{align*}
where the first is due to optimality of $\hat{Y}$ and the second is by using the triangle inequality twice and supremizing over feasible matrices. This implies that 
$$
-\mathbb{E}[\widebar{\mathcal{L}}(\Hat{Y}|Z)-\widebar{\mathcal{L}}(Y|Z)] \leq 2\underset{M' \in \Psi(\widecheck{X})}{\sup}|\widebar{\mathcal{L}}(M'|Z)-\mathbb{E}[\widebar{\mathcal{L}}(M'|Z)]|.$$
Armed with this, we will show (in Step I) that, with high probability on the randomness in $Z$,
$$\underset{M' \in \Psi(\widecheck{X})}{\sup}|\widebar{\mathcal{L}}(M'|Z)-\mathbb{E}[\widebar{\mathcal{L}}(M'|Z)]|\lesssim 
\sqrt{rd_1d_2 (d_1+d_2)\log(d_1 d_2)}$$
and (in Step II) we complete the argument by showing that
$$
\|Y-\Hat{Y}\|_F^2 \lesssim -\mathbb{E}[\widebar{\mathcal{L}}(\Hat{Y}|Z)-\widebar{\mathcal{L}}(Y|Z)].$$
To that end, we first obtain bounds for arbitrary $Y'$ before specializing to $Y'=\hat{Y}$.

 \subsection*{Step I} Since $Z=\rho(Y+G)$ and all the randomness is in $G$, we first control the deviation of $\widebar{\mathcal{L}}(M'|Z)$ from its mean.
~\\
For any positive integer $h > 0$ and a constant $\widetilde{L}_{\alpha, \sigma}$ to be determined later, by Markov’s inequality
\begin{align*}
    &\mathbb{P}(\underset{M' \in \Psi(\widecheck{X})}{\sup}|\widebar{\mathcal{L}}(M'|Z)-\mathbb{E}[\widebar{\mathcal{L}}(M'|Z)]|\geq C \widetilde{L}_{\alpha, \sigma} \alpha \sqrt{r d_1 d_2(d_1+d_2)}) \\
    \leq & \frac{\mathbb{E}[\underset{M' \in \Psi(\widecheck{X})}{\sup}|\widebar{\mathcal{L}}(M'|Z)-\mathbb{E}[\widebar{\mathcal{L}}(M'|Z)]|^h]}{(C \widetilde{L}_{\alpha, \sigma} \alpha \sqrt{r d_1 d_2(d_1+d_2)})^h}.
\end{align*}
By Lemma \ref{symmetrization}
\begin{align*}
   \mathbb{E}[\underset{M' \in \Psi(\widecheck{X})}{\sup}&|\widebar{\mathcal{L}}(M'|Z)-\mathbb{E}[\widebar{\mathcal{L}}(M'|Z)]|^h] \\
   \leq &2^h \mathbb{E}[\underset{M' \in \Psi(\widecheck{X})}{\sup}|\underset{(i,j)}{\sum}\epsilon_{ij}(\mathbbm{1}_{[Z_{ij}>0]}\frac{-1}{2\sigma^2}(M_{ij}^{'2}-2Z_{ij}M'_{ij})+\mathbbm{1}_{[Z_{ij}=0]} \log(\frac{1-f(M'_{ij})}{1-f(0)}))|^h], 
\end{align*}
where the expectation on the left is over $Z$ (equivalently $G$) and the expectation on the right is over $Z$ and the i.i.d. Rademacher random variables  $\epsilon_{i,j}$ (which are  also independent of $Z$).

To bound the right hand side, we will apply the contraction principle \ref{contraction} to the terms of the sum. Since $Z_{ij}=\rho(Y_{ij}+G_{ij})\geq 0$, we have $\mathbb{P}(Z_{ij}>t+\alpha)=\mathbb{P}(Y_{ij}+G_{ij}>t+\alpha)\leq \mathbb{P}(G_{ij}>t)$ for any $t>0$. By Lemma \ref{gaussian infty}, we have $\mathbb{P}(\max_{ij}Z_{ij}\geq \alpha+2\sqrt{\log(d_1d_2)}\sigma)\leq \frac{1}{2\sqrt{2\pi}}\frac{1}{d_1d_2\sqrt{\log(d_1d_2)}}$. When $m\in [-\alpha, \alpha]$, the function $\frac{-1}{2\sigma^2}(m^2-2Z_{ij}m)$ is Lipschitz with constant $\frac{\alpha+Z_{ij}}{\sigma^2}$ -- recall that $\alpha$ is positive and $Z_{ij}$ is non-negative. Thus, with probability at least $1-\frac{1}{2\sqrt{2\pi}}\frac{1}{d_1d_2\sqrt{\log(d_1d_2)}}$, the functions $\frac{-1}{2\sigma^2}(m^2-2Z_{ij}m)$ are Lipschitz with a uniform Lipschitz constant $2\frac{\alpha+\sqrt{\log(d_1d_2)}\sigma}{\sigma^2}$ and attain  0 when $m=0$. Similarly, the function $\log(\frac{1-f(m)}{1-f(0)})$ defined on $[-\alpha,\alpha]$ is Lipschitz with constant less than $L_{\alpha, \sigma}\leq 8 \frac{\alpha + \sigma}{\sigma^2}$ by \ref{constants} and attains 0 when $m=0$. Let $\gamma_{\alpha, \sigma} =\frac{\alpha + \sigma}{\sigma^2}$ and  let $\widetilde{L}_{\alpha, \sigma}=\max\left\{2\frac{\alpha+\sqrt{\log(d_1d_2)}\sigma}{\sigma^2}, 8 \frac{\alpha + \sigma}{\sigma^2}\right\} \leq \max\left\{2\sqrt{\log(d_1d_2)}, 8\right\} \gamma_{\alpha}$.
Thus, we showed that with probability at least $1-\frac{1}{2\sqrt{2\pi}}\frac{1}{d_1d_2\sqrt{\log(d_1d_2)}}$, the functions $\frac{1}{\widetilde{L}_{\alpha, \sigma}} \frac{-1}{2\sigma^2}(m^2-2Z_{ij}m)$ and $\frac{1}{\widetilde{L}_{\alpha, \sigma}} \log(\frac{1-f(m)}{1-f(0)})$ are contractions. Condition on this event for the moment, then by the contraction principle \ref{contraction},
\begin{align*}
\mathbb{E}[\underset{M' \in \Psi(\widecheck{X})}{\sup}&|\widebar{\mathcal{L}}(M'|Z)-\mathbb{E}[\widebar{\mathcal{L}}(M'|Z)]|^h] \\
\leq & 2^h \mathbb{E}[\underset{M' \in \Psi(\widecheck{X})}{\sup}|\underset{(i,j)}{\sum}\epsilon_{ij}(\mathbbm{1}_{[Z_{ij}>0]}\frac{-1}{2\sigma^2}(M_{ij}^{'2}-2Z_{ij}M'_{ij})+\mathbbm{1}_{[Z_{ij}=0]} \log(\frac{1-f(M'_{ij})}{1-f(0)}))|^h] \\
\leq & 2^h (2 \widetilde{L}_{\alpha, \sigma})^h \mathbb{E}[\underset{M' \in \Psi(\widecheck{X})}{\sup}|\underset{(i,j)}{\sum}\epsilon_{ij}M'_{ij}|^h] \\
\leq & (4 \widetilde{L}_{\alpha, \sigma})^h \mathbb{E}[\underset{M' \in \Psi(\widecheck{X})}{\sup}(\|E\|\|M'\|_*))^h] \\
\leq & (4 \widetilde{L}_{\alpha, \sigma})^h (\alpha \sqrt{r d_1 d_2})^h K (\sqrt{2(d_1+d_2)})^h
\end{align*}
In the last inequality, we used the nuclear norm assumption on the space $\Psi(\widecheck{X})$ and Lemma \ref{Rade norm}. Consequently, 
\begin{align*}
    \mathbb{P}(\underset{M' \in \Psi(\widecheck{X})}{\sup}&|\widebar{\mathcal{L}}(M'|Z)-\mathbb{E}[\widebar{\mathcal{L}}(M'|Z)]|\geq C \widetilde{L}_{\alpha, \sigma} \alpha \sqrt{r d_1 d_2(d_1+d_2)}) \\
    \leq & \frac{\mathbb{E}[\underset{M' \in \Psi(\widecheck{X})}{\sup}|\widebar{\mathcal{L}}(M'|Z)-\mathbb{E}[\widebar{\mathcal{L}}(M'|Z)]|^h]}{(C \widetilde{L}_{\alpha, \sigma} \alpha \sqrt{r d_1 d_2(d_1+d_2)})^h} \\
    \leq & \frac{K (4\sqrt{2} \widetilde{L}_{\alpha, \sigma} \alpha \sqrt{r d_1 d_2(d_1+d_2)})^h}{(C \widetilde{L}_{\alpha, \sigma} \alpha \sqrt{r d_1 d_2(d_1+d_2)})^h}.
\end{align*}
Setting $h\geq \log(d_1+d_2)$, the above probability is bounded by $\frac{K}{d_1+d_2}$ provided $C\geq 4\sqrt{2}e$. So, accounting for the event we conditioned on, we now have $$\mathbb{P}(\underset{M' \in \Psi(\widecheck{X})}{\sup}|\widebar{\mathcal{L}}(M'|Z)-\mathbb{E}[\widebar{\mathcal{L}}(M'|Z)]|\geq C \widetilde{L}_{\alpha, \sigma} \alpha \sqrt{r d_1 d_2(d_1+d_2)})\leq \frac{K}{d_1+d_2}+\frac{1}{2\sqrt{2\pi}}\frac{1}{d_1d_2\sqrt{\log(d_1d_2)}}.$$
\subsection*{Step II} The ground truth is $Y \in \Psi(\widecheck{X})$, and for any $Y' \in \Psi(\widecheck{X})$ it holds that
\begin{align*}
   \widebar{\mathcal{L}}(Y'|Z)-\widebar{\mathcal{L}}(Y|Z)
   &= \mathbb{E}[\widebar{\mathcal{L}}(Y'|Z)-\widebar{\mathcal{L}}(Y|Z)]+(\widebar{\mathcal{L}}(Y'|Z)-\mathbb{E}[\widebar{\mathcal{L}}(Y'|Z)])-(\widebar{\mathcal{L}}(Y|Z)-\mathbb{E}[\widebar{\mathcal{L}}(Y|Z)])\\
   & \leq  \mathbb{E}[\widebar{\mathcal{L}}(Y'|Z)-\widebar{\mathcal{L}}(Y|Z)]+2\underset{M' \in \Psi(\widecheck{X})}{\sup}|\widebar{\mathcal{L}}(M'|Z)-\mathbb{E}[\widebar{\mathcal{L}}(M'|Z)]|
\end{align*}

Our remaining goal is then to control $-\mathbb{E}[\widebar{\mathcal{L}}(Y'|Z)-\widebar{\mathcal{L}}(Y|Z)]$, where $\widebar{\mathcal{L}}(Y'|Z)=\mathcal{L}(Y'|Z)-\mathcal{L}(\mathbf{0}|Z)=\underset{(i,j)}{\sum}(\mathbbm{1}_{[Z_{ij}>0]}\frac{-1}{2\sigma^2}(Y_{ij}^{'2}-2Z_{ij}Y'_{ij})+\mathbbm{1}_{[Z_{ij}=0]} \log(\frac{1-f(Y'_{ij})}{1-f(0)}))$. To that end, note that $\mathbb{P}(Z_{ij}>0)=\mathbb{P}(Y_{ij}+G_{ij}>0)=\mathbb{P}(G_{ij}>-Y_{ij})=f(Y_{ij})$. When $Z_{ij}>0$, we have $Z_{ij}=\rho(Y_{ij}+G_{ij})=Y_{ij}+G_{ij}$. Thus $(Y_{ij}^{'2}-2Z_{ij}Y'_{ij})-(Y_{ij}^2-2Z_{ij}Y_{ij})=
(Y'_{ij}-Y_{ij})^2-2(Y'_{ij}-Y_{ij})G_{ij}$. Substituting this into the expression for $-\mathbb{E}[\widebar{\mathcal{L}}(Y'|Z)-\widebar{\mathcal{L}}(Y|Z)]$ to obtain
\begin{align}\label{eq:three_terms}
    -\mathbb{E}[\widebar{\mathcal{L}}(Y'|Z)-\widebar{\mathcal{L}}(Y|Z)] 
    = \underset{(i,j)}{\sum}f(Y_{ij})\frac{1}{2\sigma^2}&(Y'_{ij}-Y_{ij})^2 +\underset{(i,j)}{\sum}\frac{1}{\sigma^2}(Y_{ij}-Y'_{ij})\mathbb{E}[\mathbbm{1}_{[G_{ij}>-Y_{ij}]}G_{ij}]\\ &+\underset{(i,j)}{\sum}(1-f(Y_{ij}))\log(\frac{1-f(Y_{ij})}{1-f(Y'_{ij})}).\nonumber
\end{align}
To simplify this expression, note that 
$
    \mathbb{E}[\mathbbm{1}_{[G_{ij}>-Y_{ij}]}G_{ij}]
    = \frac{\sigma}{\sqrt{2\pi}}e^{-{\frac{Y_{ij}^2}{2\sigma^2}}}$,
so $$\underset{(i,j)}{\sum}\frac{1}{\sigma^2}(Y_{ij}-Y'_{ij})\mathbb{E}[\mathbbm{1}_{[G_{ij}>-Y_{ij}]}G_{ij}]=\underset{(i,j)}{\sum}\frac{1}{\sigma^2}(Y_{ij}-Y'_{ij})\frac{\sigma}{\sqrt{2\pi}}e^{-{\frac{Y_{ij}^2}{2\sigma^2}}}=\underset{(i,j)}{\sum}f'(Y_{ij})(Y_{ij}-Y'_{ij}).$$

Next we deal with the last summand in \eqref{eq:three_terms}, which satisfies 
\begin{align*}
    \underset{(i,j)}{\sum}(1-f(Y_{ij}))&\log(\frac{1-f(Y_{ij})}{1-f(Y'_{ij})}) 
    =\underset{(i,j)}{\sum}D_{KL}(f(Y_{ij}), f(Y'_{ij}))-\underset{(i,j)}{\sum}f(Y_{ij})\log(\frac{f(Y_{ij})}{f(Y'_{ij})})\\
    &=d_1d_2 D_{KL}(f(Y)\|f(Y'))+\underset{(i,j)}{\sum}f(Y_{ij})(\log(f(Y'_{ij}))-\log(f(Y_{ij})))\\
    &\geq d_1d_2 D_{KL}(f(Y)\|f(Y'))+\underset{(i,j)}{\sum}f(Y_{ij})[\frac{f'(Y_{ij})}{f(Y_{ij})}(Y'_{ij}-Y_{ij})-\frac{1}{2\sigma^2}(Y'_{ij}-Y_{ij})^2]
\end{align*}
In the last step we used Corollary \ref{logtalyor}. Thus, using Lemma \ref{HelltoKL} and Lemma \ref{HelltoFrob},
\begin{align*}
    -\mathbb{E}[\widebar{\mathcal{L}}(Y'|Z)-\widebar{\mathcal{L}}(Y|Z)] 
    &\geq \underset{(i,j)}{\sum}f(Y_{ij})\frac{1}{2\sigma^2}(Y'_{ij}-Y_{ij})^2+\underset{(i,j)}{\sum}f'(Y_{ij})(Y_{ij}-Y'_{ij})+d_1d_2 D_{KL}(f(\mathbf{Y})\|f(\mathbf{\Hat{Y}}))\\
    & \qquad \qquad +\underset{(i,j)}{\sum}f(Y_{ij})[\frac{f'(Y_{ij})}{f(Y_{ij})}(Y'_{ij}-Y_{ij})-\frac{1}{2\sigma^2}(Y'_{ij}-Y_{ij})^2]\\
    & =  d_1d_2 D_{KL}(f(Y)\|f(Y'))
    \geq  d_1d_2 d_H^2(f(Y),f(Y'))\\
    & \geq  d_1d_2 \frac{1}{8\beta_{\alpha, \sigma}} \frac{\|Y-Y'\|_F^2}{d_1d_2}
    =  \frac{1}{8\beta_{\alpha, \sigma}}\|Y-Y'\|_F^2.
\end{align*}

Now, we can apply this with the choice $Y'=\Hat{Y}$, the maximizer of ($P_*$) and use the fact that $\widebar{\mathcal{L}}(\Hat{Y}|Z)\geq \widebar{\mathcal{L}}(Y|Z)$ to deduce that
$$
\frac{1}{8\beta_{\alpha, \sigma}}\|Y-\Hat{Y}\|_F^2 \leq -\mathbb{E}[\widebar{\mathcal{L}}(\Hat{Y}|Z)-\widebar{\mathcal{L}}(Y|Z)]\leq 2\underset{M' \in \Psi(\widecheck{X})}{\sup}|\widebar{\mathcal{L}}(M'|Z)-\mathbb{E}[\widebar{\mathcal{L}}(M'|Z)]|
.$$

Thus, we have with probability at least $1-(\frac{K}{d_1+d_2}+\frac{1}{2\sqrt{2\pi}}\frac{1}{d_1d_2\sqrt{\log(d_1d_2)}})$, $\|Y-\Hat{Y}\|_F^2 \leq (8\beta_{\alpha, \sigma})2C \widetilde{L}_{\alpha, \sigma} \alpha \sqrt{r d_1 d_2(d_1+d_2)}) \leq 16C\alpha\beta_{\alpha, \sigma}\gamma_{\alpha, \sigma} max\{2\sqrt{\log(d_1d_2)}, 8\} \sqrt{r d_1 d_2(d_1+d_2)}$, where $C$ is an absolute constant. Denote $C_{\alpha, \sigma}:=16C\alpha\beta_{\alpha, \sigma}\gamma_{\alpha, \sigma}$. We can then rewrite this as
\begin{equation*}
\frac{1}{d_1d_2}\|Y-\Hat{Y}\|_F^2 \leq C_{\alpha, \sigma} \max\{2\sqrt{\log(d_1d_2)}, 8\} \sqrt{\frac{r(d_1+d_2)}{d_1 d_2}} ,
\end{equation*} which concludes our proof.
\end{proof}

\section{Connection to Frobenius norm minimization }\label{sec:con to frob}

Here, we show that solving  (\ref{opt P star}) is equivalent to minimizing a tight convex upper bound on $\frac{1}{2}\|Z-\rho(M')\|_F^2$. This is, for example, analogous to the common practice of maximizing the evidence lower bound (ELBO) \cite{luo2022understanding} as a lower bound on the log-likelihood  for an unknown data distribution.
In our case, consider the natural albeit non-convex optimization problem 
$$
   \underset{M'}{\text{minimize}} \ \frac{1}{2}\|Z-\rho(M')\|_F^2, \ \ \text{subject to} \ M' \in \Psi(\widecheck{X}),$$ 
and note that 
$
    \frac{1}{2}\|Z-\rho(M')\|_F^2 = \underset{(i,j):Z_{ij}>0}{\sum}\frac{1}{2}(\rho(M'_{ij})-Z_{ij})^2+\underset{(i,j):Z_{ij}=0}{\sum}\frac{1}{2}\rho(M'_{ij})^2
$ is non-convex and non-differentiable as $\frac{1}{2}(\rho(x)-c)^2$ is non-convex and non-differentiable for any positive constant $c$. One way around this is to replace $\underset{(i,j):Z_{ij}>0}{\sum}\frac{1}{2}(\rho(M'_{ij})-Z_{ij})^2$ by its tight upper bound$\underset{(i,j):Z_{ij}>0}{\sum}\frac{1}{2}(M'_{ij}-Z_{ij})^2$ 
and $\frac{1}{2}\rho(x)^2$ by its tight upper bound $-\sigma^2 \log(1-f(x))$, where the tightness is established by \cref{asymptotic property} below, and $f(x)$ is the CDF associated with $\mathcal{N}(0, \sigma^2)$. These relaxations yield an equivalent form of (\ref{opt P star}):
$$
   \underset{M'}{\text{minimize}} \  \underset{(i,j):Z_{ij}>0}{\sum}\frac{1}{2}(Z_{ij}-M'_{ij})^2+\underset{(i,j):Z_{ij}=0}{\sum}-\sigma^2\log(1-f(M'_{ij})) 
   \text{ subject to} \ M' \in \Psi(\widecheck{X}).
$$

\begin{lemma}\label{asymptotic property}(Tightness of Relaxation)~\\
Let $f(x)$ be the CDF of normal distribution $\mathcal{N}(0, \sigma^2)$, then the function $-\sigma^2 \log(1-f(x))$ is asymptotically $\frac{1}{2}x^2$ as $x\rightarrow \infty$. 
\begin{proof}
    By repeated application of L'Hopital's Rule, we have:
    \begin{align*}
        &\lim_{x \to \infty}\frac{-\sigma^2 \log(1-f(x))}{\frac{1}{2}x^2}=\lim_{x \to \infty}\frac{\sigma^2f'(x)}{x(1-f(x))}\\
        \stackrel{\text{$y=\frac{x}{\sigma}$}}{\scalebox{3}[1]{=}}&\lim_{y \to \infty}\frac{\phi(y)}{y(1-\Phi(y))}=1\\
    \end{align*}
We used $f(x)=\Phi(\frac{x}{\sigma})$, $f'(x)=\frac{1}{\sigma}\phi(\frac{x}{\sigma})$ and  $\phi'(x)=-x\phi(x)$.
\end{proof}
\end{lemma}

\if{The following plots \ref{fig1} \ref{fig2} display the functions $-\sigma^2 \log(1-f(x))$ and $\frac{1}{2}\rho(x)^2$ with $\sigma=1$. 
\begin{figure}[htbp]
    \centering
    \begin{subfigure}[b]{0.45\textwidth}
        \centering
        \includegraphics[width=\textwidth]{fig1.png}
        \caption{Asymptotic Behavior}
        \label{fig1}
    \end{subfigure}
    \hfill
    \begin{subfigure}[b]{0.45\textwidth}
        \centering
        \includegraphics[width=\textwidth]{fig2.png}
        \caption{Close to Zero Behavior}
        \label{fig2}
    \end{subfigure}
\end{figure}

Figure \ref{fig1} is in accordance with lemma \ref{asymptotic property}. Figure \ref{fig2} shows that 
$-\sigma^2 \log(1-f(x))$ decreases to negligible at $x\approx -2\sigma$ which penalizes a little more than $\frac{1}{2}\rho(x)^2$ does according to the noise level.
Since $\frac{1}{2}\rho(x)^2 \leq C[-\sigma^2 \log(1-f(x))]$ won't hold for any positive $C<1$, we're replacing $\frac{1}{2}\rho(x)^2$ by a tight upper bound. }\fi

\if{\begin{lemma}\label{relaxation lemma}
    Let $h$ and $g$ be two functions defined on the same convex bounded domain $\Omega \subset \mathbb{R}^d$ with smooth boundary. Let $h: \mathbb{R}^d \rightarrow \mathbb{R}$ be $m$-strongly convex and $h \in \mathcal{C}^1$. Let $g: \mathbb{R}^d \rightarrow \mathbb{R}$ be convex and $g \in \mathcal{C}^1$.  Assume $\underset{x \in \Omega}{\sup}\|\nabla h - \nabla g\|\leq \epsilon$. Let $x^*$ be the (unique) minimizer to $\underset{x \in \Omega}{\min} \ h(x)$ and $x^{\#}$ be a minimizer of $\underset{x \in \Omega}{\min} \ g(x)$. Then we have $\|x^* - x^{\#}\|\leq \frac{\epsilon}{m}$.
\end{lemma}
\begin{proof}
    By assumption, $\|\nabla h(x^{\#}) - \nabla g(x^{\#})\|\leq \epsilon$ and $\nabla g(x^{\#})=0$ by the first order necessary condition. We can get $\|\nabla h(x^{\#})\|\leq\epsilon$. \\
    From strong convexity of $h$, one can deduce that $m\|x^* - x^{\#}\| \leq \|\nabla h(x^*) - \nabla h(x^{\#})\|$. As $\nabla h(x^*)=0$ we get $m\|x^* - x^{\#}\| \leq \|\nabla h(x^{\#})\| \leq \epsilon$, which finishes the proof.
\end{proof}
~\\
In our setting, take $\Omega=\Psi \subset \mathbb{R}^{d_1d_2}$, $h(M)=\underset{(i,j):Z_{ij}>0}{\sum}\frac{1}{2}(Z_{ij}-M_{ij})^2+\underset{(i,j):Z_{ij}=0}{\sum}-\sigma^2 \log(1-f(M_{ij}))$ and $g(M)=\underset{(i,j):Z_{ij}>0}{\sum}\frac{1}{2}(Z_{ij}-M_{ij})^2+\underset{(i,j):Z_{ij}=0}{\sum}\frac{1}{2}\rho(M_{ij})^2$. Here we abused the use of $\mathbb{R}^{d_1d_2}$ and $\mathbb{R}^{d_1 \times d_2}$. Both constants $\epsilon$ and $m$ will depend on $\Psi$ and noise $G$, i.e. $\alpha$ and $\sigma$.
}\fi

\if{To be specific, if we divide the objective function in ($P_{**}$) by $\sigma^2$ and add a negative sign to switch the minimization problem to maximization problem we'll get
$$
   \underset{M'}{\text{maximize}} \ \underset{(i,j):Z_{ij}>0}{\sum}\log(e^{-\frac{(Z_{ij}-M'_{ij})^2}{2\sigma^2}})+\underset{(i,j):Z_{ij}=0}{\sum}\log(1-f(M'_{ij})), \ 
   \text{subject to} \ M' \in \Psi(\widecheck{X}).  
$$
Since adding a constant to the objective function doesn't change the optimizer, we can add $\frac{1}{\sqrt{2\pi}\sigma}$ into the logarithm and get (\ref{opt P star}) back.
}\fi
\end{document}